\documentclass[journal]{IEEEtran}

\ifCLASSINFOpdf
\else
\fi

\usepackage{amssymb}
\usepackage{amsmath}
\usepackage{subfigure, graphicx}
\usepackage{epstopdf}
\usepackage{tikz}
\usepackage{microtype}
\usepackage{hyperref}
\usepackage{amsfonts}
\usepackage{booktabs}
\usepackage{float}
\usepackage{setspace}
\usepackage{algorithm}
\usepackage{algorithmicx}
\usepackage{tabularx}
\usepackage{algpseudocode}
\usepackage{amsthm}%
\usepackage{mathrsfs}%
\usepackage{multirow}
\usepackage{array}
\usepackage{cite}
\usepackage{bm}
\usepackage{pifont}
\usepackage{epstopdf}
\usepackage{bbold}
\usepackage{newtxmath}
\usepackage{bbm}
\usepackage{dsfont}
\usepackage{mathbbol}
\usepackage{mathtools}
\usepackage{cleveref}

\usepackage{tikz}
\usetikzlibrary{shadows}
\usetikzlibrary{arrows,positioning}

\makeatletter
\newcommand{\multiline}[1]{%
	\begin{tabularx}{\dimexpr\linewidth-\ALG@thistlm}[t]{@{}X@{}}
		#1
	\end{tabularx}
}
\makeatother

\floatname{algorithm}{Algorithm}

\begin{document}

\title{On the Global Solution of Soft $k$-Means}

\author{Feiping~Nie,~\IEEEmembership{Senior Member, IEEE}, Hong~Chen, Rong Wang, and~Xuelong~Li,~\IEEEmembership{Fellow,~IEEE}
\thanks{Feiping Nie and Hong Chen are with the School of Computer Science, School of Artificial Intelligence, Optics and Electronics (iOPEN), and the Key Laboratory of Intelligent Interaction and Applications (Ministry of Industry and Information Technology), Northwestern Polytechnical University, Xi'an, Shaanxi 710072, China. (e-mail: feipingnie@gmail.com; hochen1017@gmail.com). (\textit{Corresponding author: Feiping Nie}.)
	
Rong Wang and Xuelong Li are with the School of Artificial Intelligence, Optics and Electronics (iOPEN), and the Key Laboratory of Intelligent Interaction and Applications (Ministry of Industry and Information Technology), Northwestern Polytechnical University, Xi'an, Shaanxi 710072, China. (e-mail: wangrong07@tsinghua.org.cn; li@nwpu.edu.cn).}}

\markboth{}%
{Shell \MakeLowercase{\textit{et al.}}: On the Global Solution of Soft $k$-Means}

\maketitle

\begin{abstract}
This paper presents an algorithm to solve the Soft $k$-Means problem globally.
Unlike Fuzzy $c$-Means, Soft $k$-Means (S$k$M) has a matrix factorization-type objective and 
has been shown to have a close relation with the popular probability decomposition-type clustering methods, 
e.g., Left Stochastic Clustering (LSC).
Though some work has been done for solving the Soft $k$-Means problem,
they usually use an alternating minimization scheme or the projected gradient descent method, 
which cannot guarantee global optimality since the non-convexity of S$k$M.
In this paper, we present a sufficient condition for a feasible solution of Soft $k$-Means problem to be globally optimal
and show the output of the proposed algorithm satisfies it.
Moreover, for the Soft $k$-Means problem, we provide interesting discussions on 
stability, solutions non-uniqueness, and connection with LSC.
Then, a new model, named Minimal Volume Soft $k$-Means (MVS$k$M), is proposed to address the solutions non-uniqueness issue.
Finally, experimental results support our theoretical results.
\end{abstract}


\IEEEpeerreviewmaketitle

\section{Introduction}
\IEEEPARstart{K}{means} \cite{lloyd1982least}, arguably most commonly used techniques for data analysis \cite{wu2008top}, produces a clustering such that the sum of squared error between samples and the mean of their cluster is minimized. For $n$ feature vectors gathered as the $d$-dimensional columns of a matrix $\mathbf{X}\in\mathbb{R}^{d\times n}$, the $k$-Means problem can be written as a matrix factorization problem \cite{bauckhage2015k}:
\[
\min_{\mathbf{F},\mathbf{G}\in\{0,1\}^{n,k},\mathbf{G}\mathbb{1}_k = \mathbb{1}_n} \| \mathbf{X} - \mathbf{FG}^\top\|_F^2,
\]
where $\mathbf{X}\in\mathbb{R}^{d\times n}$ is the data matrix, $\mathbf{F}\in\mathbb{R}^{d\times k}$ is the prototypes, and $\mathbf{G}\in\mathbb{R}^{k\times n}$ is the membership indicator matrix. 
Several clustering methods have been proposed by relaxing this matrix factorization $k$-Means problem in different aspects.
For example, \cite{zha2002spectral,ding2005equivalence} relaxed the structural constraints on indicator matrix $\mathbf{G}$ to an orthogonality constraint;
\cite{ding2010convex} changed the constraints on $\mathbf{G}$ in the $k$-Means problem to only require that $\mathbf{G}$ be positive.

Sometimes, dividing data into distinct clusters is too strict, where each data point can only belong to exactly one cluster. In fuzzy clustering, data points can potentially belong to multiple clusters.
The probably best known approach to fuzzy clustering is the method of Fuzzy $c$-Means (FCM) proposed by \cite{dunn1973fuzzy}.
The objective function of FCM is
\[
\min_{\mathbf{F},\mathbf{G}\succeq 0,\mathbf{G}\mathbb{1}_k = \mathbb{1}_n} \sum_{i=1}^n \sum_{j=1}^k g_{ij}^m \| \mathbf{x}_i - \mathbf{f}_j\|_2^2.
\]
In FCM, membership functions are defined based on a distance function, and membership degrees express proximities of entities to cluster centers (i.e., prototypes $\mathbf{F}$). By choosing a suitable distance function different cluster shapes can be identified. Another approach to fuzzy clustering due to \cite{krishnapuram1993possibilistic} is the possibilistic $c$-Means (PCM) algorithm which eliminates one of the constraints imposed on the search for $c$ partitions leading to possibilistic absolute fuzzy membership values instead of FCM probabilistic relative fuzzy memberships.

Another potentially interesting relaxation of the matrix factorization form $k$-Means is to relax the $\mathbf{G}\in\{0,1\}^{k\times n}$
to $\mathbf{G}\in\mathbb{R}^{k\times n}$. This type of relaxation is popular in submodular optimization and is closely related to the Lov\'{a}sz extension \cite{bach2013learning}. The objective function we concerned is
\begin{equation}\label{eq:SfM}
\min_{\mathbf{F},\mathbf{G}\succeq 0,\mathbf{G}\mathbb{1}_k = \mathbb{1}_n} \| \mathbf{X} - \mathbf{FG}^\top\|_F^2,
\end{equation}
which is the problem we want to solve in this paper and has been termed Soft $k$-Means (S$k$M) in \cite{arora2013similarity}.
Another line of research that is related to Problem \eqref{eq:SfM} is the fuzzy clustering with proportional membership (FCPM) scheme \cite{nascimento2003modeling} in the literature.
The idea of FCPM is to develop explicit mechanisms for data generation from cluster structures, since such a model can provide a theoretical framework for cluster structures found in data. Indeed, the Soft $k$-Means problem is exactly equivalent to the FCPM problem when the fuzzy factor in FCPM is set to zero.

Yet some work \cite{nascimento2000fuzzy,nascimento2003modeling,nascimento2016applying} has been done for solving the Soft $k$-Means problem in Problem \eqref{eq:SfM},
they usually use an alternating minimization scheme or the projected gradient descent method, 
which cannot guarantee global optimality since the non-convexity of S$k$M.

In this paper, we present a sufficient condition for a feasible solution of Soft $k$-Means problem to be globally optimal
and show the output of the proposed algorithm satisfies it.
Moreover, for the Soft $k$-Means problem, we provide interesting discussions on 
stability, solutions non-uniqueness, and connection with LSC.
Then, a new model, named Minimal Volume Soft $k$-Means (MVS$k$M), is proposed to address the solutions non-uniqueness issue.
Finally, experimental results support our theoretical results.

\textbf{Contributions.} The main contributions of this paper are summarized as follows:
\begin{itemize}
	\item We report sufficient condition under which $\mathbf{F}$ and $\mathbf{G}$ are a global solution of the Soft $k$-Means problem. As a byproduct, we provide sufficient and necessary conditions for a data matrix to be S$k$Mable.
	\item We propose a simple algorithm with computational complexity $\mathcal{O}(nd^2)$. The output of this algorithm satisfies the aforementioned sufficient condition, thus it is a global solution of the Soft $k$-Means problem with theoretical guarantee.
	\item Using the sufficient condition, we provide interesting discussions on the stability, optimal non-uniqueness, and connection with Left Stochastic Clustering proposed in \cite{arora2013similarity}. Moreover, we proposed an model, named Minimal Volume Soft $k$-Means to address the optimal non-uniqueness problem.
\end{itemize}

\textbf{Organization.} 
The rest of this paper is organized as follows. 
We review some closely related prior work in \Cref{sec:prior}.
Notations description and the proposed algorithm to globally solve Soft $k$-Means are put in \Cref{sec:algorithm}.
The main theory which gives the sufficient optimality condition is given in \Cref{sec:theory} along with its proof sketch.
Several interesting discussions on stability, solutions non-uniqueness, and connection with LSC in \Cref{sec:discuss}.
Experimental results on both synthetic dataset and real-world datasets are provided in \Cref{sec:exps}.
Finally, conclude the paper in \Cref{sec:conclusion}.

\section{Prior Work}\label{sec:prior}
In this section, we review methods which are closely related to the Soft $k$-Means problem.

\subsection{NMF and $k$-Means}
This line of research work includes Convex-NMF, Symmetric-NMF, etc.
A note on matrix factorization representation is in \cite{bauckhage2015k}.
\cite{zha2002spectral} proposed  relaxing the constraints on $\mathbf{G}$ in the $k$-Means optimization problem to an orthogonality constraint:
\[
\min_{\mathbf{F},\mathbf{G}\succeq 0,\mathbf{G}^\top\mathbf{G} = \mathbb{I}_{k}} \| \mathbf{X} - \mathbf{FG}^\top\|_F^2.
\]
\cite{ding2005equivalence} considered the kernelized clustering objective:
\[
\min_{\mathbf{G}\succeq 0,\mathbf{G}^\top\mathbf{G} = \mathbb{I}_{k}} \| \mathbf{K} - \mathbf{GG}^\top\|_F^2.
\]
Then, \cite{ding2010convex} considered changing the constraints on $\mathbf{G}$ in the $k$-Means problem to only require that $\mathbf{G}$ be positive. They explored a number of approximations to the $k$-Means problem that imposed different constraints on $\mathbf{F}$. One such variant that they deemed particularly worthy of further investigation was convex NMF. Convex NMF restricts the columns of $\mathbf{F}$ (the cluster centroids) to be convex combinations of the columns of $\mathbf{X}$:
\[
\min_{\mathbf{
		W}\succeq 0, \mathbf{G} \succeq 0} \| \mathbf{X} - \mathbf{XWG}^\top\|_F^2.
\]
\subsection{Fuzzy Proportional Membership Clustering}
The fuzzy clustering with proportional membership (FCPM) \cite{nascimento2000fuzzy,nascimento2003modeling,nascimento2016applying}. scheme considers explicit mechanisms for data generation from cluster structures.
The objective for general FCPM$_m$ problem is
\[
\min_{\mathbf{F}, \mathbf{G}\succeq 0,\mathbf{G}\mathbb{1}_k = \mathbb{1}_n} 
\sum_{i=1}^n\sum_{j=1}^k g_{ij}^m \| \mathbf{x}_i - \mathbf{f}_j g_{ij} \|_2^2.
\]
It is notable that when for $m=0$, FCPM$_0$ is equivalent to Soft $k$-Means problem in Problem \eqref{eq:SfM}.
Thus, the algorithm for solving Problem \eqref{eq:SfM} can be used to solve FCPM$_0$.

\subsection{Left Stochastic Clustering}
Left Stochastic Clustering (LSC) proposed by \cite{arora2013similarity} provide directly left stochastic similarity matrix factorization.
They perform clustering by solving a non-negative matrix factorization problem for the best cluster assignment matrix $\mathbf{G}$. That is, given a matrix $\mathbf{K}$, clustering by finding a scaling factor $c$ and a cluster probability matrix $\mathbf{G}$ that best solve
\begin{equation}\label{eq:LSD}
\min_{c\geq 0, \mathbf{G}\succeq 0,\mathbf{G}\mathbb{1}_k = \mathbb{1}_n} \left\| \mathbf{K} - \frac{1}{c}\mathbf{GG}^\top \right\|_F^2.
\end{equation}
It has been shown that the Problem \eqref{eq:LSD} is closely related to Soft $k$-Means (Theorem \ref{thm:jmlr13prop}).
Specifically, if the Gram of given matrix is LSDable then this matrix is S$k$Mable (see details in \Cref{sec:leftsto}).

\section{Notations and Algorithm} \label{sec:algorithm}
In this section, we introduce the notations used in the paper and report the proposed algorithm for solving the soft $k$-Means problem.
Some remarks just follows.
\subsection{Notations}
Throughout this paper, scalars, vectors and matrices are denoted by lowercase letters, boldface lowercase letters and boldface uppercase letters, respectively; for a matrix $\mathbf{A} \in \mathbb{R}^{n \times n}$, $\mathbf{A}^\top$ denotes the transpose of $\mathbf{A}$, $\text{Tr}(\mathbf{A}) = \sum_{i=1}^n a_{ii}$, $\| \mathbf{A} \|_F^2 = \text{Trace}(\mathbf{A}^\top\mathbf{A})$; $\mathbb{1}_n\in\mathbb{R}^n$ denotes vector with all ones; $\|\mathbf{x}\|_0$ denotes the number of non-zero elements; $\|\mathbf{A}\|_{p,q} = \left( \sum_{i=1}^n \|\mathbf{a}\|_p^q \right)^{1/q}$; $\mathbb{I}_n \in \mathbb{R}^{n\times n}$ denotes the identity matrix; $\mathbf{H}_n = \mathbb{I}_{n} - \frac{1}{n}\mathbb{1}_n\mathbb{1}_n^\top$ is the centralization matrix; $\mathbf{A}\succeq 0$ indicates all elements of $\mathbf{A}$ are non-negative; diag$(\mathbf{A})$ is the 
diagonal elements of $\mathbf{A}$; diag$(\mathbf{v})$ is a matrix with diagonal elements $\mathbf{v}$.

\subsection{Algorithm}
In this subsection, we provide an algorithm to solve the soft $k$-Means Problem \eqref{eq:SfM}.
In \Cref{sec:mainTheory}, we will show that the output prototype $\mathbf{F}$ and membership $\mathbf{G}$ of Algorithm \ref{alg:global} is a global solution of Problem \eqref{eq:SfM}.
\begin{algorithm}[H]  
	\caption{Solve the Soft $k$-Means Problem \eqref{eq:SfM}}  
	\label{alg:global}
	\begin{algorithmic}[1]  
		\Require $\mathbf{X}\in\mathbb{R}^{d\times n}$ and the clusters number $k$
		\State $\overline{\mathbf{x}} \leftarrow \frac{1}{n}\mathbf{X}\mathbb{1}_n$
		\State $\mathbf{U}_{k-1}\leftarrow$  the $(k-1)$-truncated left SVs of $\mathbf{X}$
		\State $\mathbf{B} \leftarrow$ the $(k-1)$-truncated EVs of $\mathbf{H}_k$
		\State Ensure $r \geq \|\mathbf{U}_{k-1}^\top\mathbf{X} \|_{2,\infty}$
		\State $\mathbf{F}\leftarrow r\sqrt{k(k-1)}\mathbf{U}_{k-1}\mathbf{B}^\top + \overline{\mathbf{x}}\mathbb{1}_k^\top$
		\State $\mathbf{G}^\top \leftarrow \frac{1}{r\sqrt{k(k-1)}}\mathbf{BU}_{k-1}^\top \mathbf{X} + \frac{1}{k}\mathbb{1}_k\mathbb{1}_n^\top$
		\Ensure Prototypes $\mathbf{F} \in \mathbb{R}^{d\times k}$, membership $\mathbf{G} \in \mathbb{R}^{k \times n}$
		
		\label{code:recentEnd}  
	\end{algorithmic}  
\end{algorithm} 
\newtheorem{remark}{\bf Remark}
\begin{remark} Some remarks just follows.
	
	\begin{itemize}
		\item The computational complexity of Algorithm \ref{alg:global} is $\mathcal{O}(nkd)$, which is linear with respect to $n$, 
		which shows that our algorithm is suitable for large-scale data sets.
		\item It is notable that, unlike existing method to solve Problem \eqref{eq:SfM}, the newly proposed Algorithm \ref{alg:global} is non-iterative and it is easily implementable with few lines of basic linear algebra manipulations.
		While its simpleness, in \Cref{sec:mainRes}, we will show the output of Algorithm \ref{alg:global} is a globally solution of Problem \eqref{eq:SfM}.
		\item We should point out that the output of Algorithm \ref{alg:global} is not unique.
		The reason is that the $(k-1)$-truncated EVs of $\mathbf{H}_k$, denoted by $\mathbf{B}$, are not unique since the leading $(k-1)$ eigenvalues of $\mathbf{H}_k$ are equal. But the output $\mathbf{F}$ and $\mathbf{G}$ are still optimal, which reveals the solution non-uniqueness issue of the Soft $k$-Means problem.
		This issue would be mentioned again in \Cref{remark:proof_sketch_unique} and clearly explained and discussed in \Cref{sec:arbitray}.
	\end{itemize}
	
	
	
\end{remark}
%
\section{Main Theory} \label{sec:mainTheory} \label{sec:theory}
In this section, we present our main results on the sufficient condition for $\mathbf{F}$ and $\mathbf{G}$ to globally solve the Problem \eqref{eq:SfM}. 
Then, we show the output returned by Algorithm \ref{alg:global} satisfies the condition.
Finally, the proof sketch for the main result is provided in \Cref{sec:proofSketch}.

\subsection{Our Results}\label{sec:mainRes}
\newtheorem{thm}{\bf Theorem}
\begin{thm}\label{thm:mainRes}
	Let the centralized data $\overline{\mathbf{X}} \in\mathbb{R}^{d\times n}$ be $\mathbf{X}\mathbf{H}_n$.
	Let the $(k-1)$-truncated SVD of the centralized data matrix be $\mathbf{U}_{k-1}\mathbf{\Sigma}_{k-1}\mathbf{V}_{k-1}^\top$. 
	If there exists $\mathbf{S}\in\mathbb{R}^{(k-1)\times k}$ and $\mathbf{G}\in\mathbb{R}^{k\times n}, \mathbf{G}\succeq 0, \mathbf{G}\mathbb{1}_k = \mathbb{1}_n$ satisfy
	\[
	\mathbf{\Sigma}_{k-1}\mathbf{V}_{k-1}^\top = \mathbf{S}\mathbf{G}^\top.
	\]
	Then, $\mathbf{F}=\mathbf{U}_{k-1}\mathbf{S} + \frac{1}{n}\mathbf{X}\mathbb{1}_n\mathbb{1}_k^\top$ and $\mathbf{G}$ is a solution of Problem \eqref{eq:SfM}.
\end{thm}

\begin{remark}
	In Theorem \ref{thm:mainRes}, we report a sufficient condition for $\mathbf{F}$ and $\mathbf{G}$ to be a solution of Problem \eqref{eq:SfM}.
	In detail, we require that there exists $\mathbf{S}\in\mathbb{R}^{(k-1)\times k}$ and $\mathbf{G}$ satisfy (a) $\mathbf{\Sigma}_{k-1}\mathbf{V}_{k-1}^\top = \mathbf{S}\mathbf{G}^\top$; (b) $\mathbf{G}\mathbb{1}_k = \mathbb{1}_n$; (c) $\mathbf{G} \succeq 0$.
	In the following corollary, we justify the existence of such $\mathbf{F}$ and $\mathbf{G}$ by showing the Algorithm \ref{alg:global} produces
	a pair of $\mathbf{F}$ and $\mathbf{G}$ satisfying the conditions in Theorem \ref{thm:mainRes}.
\end{remark}

\newtheorem{corollary}{\bf Corollary}
\begin{corollary}\label{thm:mainCorollary}
	Let $\mathbf{F}^*, \mathbf{G}^*$ be the output of Algorithm \ref{alg:global}. We have
	\[
	\mathbf{F}^*, \mathbf{G}^* \in \mathop{\arg\min}_{\mathbf{F},\mathbf{G}\succeq 0,\mathbf{G}\mathbb{1}_k = \mathbb{1}_n} \| \mathbf{X} - \mathbf{FG}^\top\|_F^2.
	\]
\end{corollary}

The Corollary \ref{thm:mainCorollary} justifies the merit of the proposed Algorithm \ref{alg:global}, thus we prove it here. Before the proof of this main corollary, we need following numerical estimation.
\newtheorem{Lemma1}{\bf Lemma}
\begin{Lemma1}\label{lem:numerical}
	Let $\mathbf{x} \in \mathbb{R}^k$.
	If $\mathbb{1}_k^\top\mathbf{x}=0$ and $\|\mathbf{x}\|_2 \leq 1$, then we have $\|\mathbf{x}\|_\infty \leq \frac{\sqrt{k(k-1)}}{k}.$
\end{Lemma1}

\begin{figure*}
	\center
	\begin{tikzpicture}[
	block/.style={
		thick,
		draw,
		drop shadow={opacity=0.75},
		fill=white,
		minimum height=0.7cm,
		text centered
	},
	pil/.style={
		thick,
		->,
		in=180,
		out=0,
		shorten <=2pt,
		shorten >=1pt
	},
	npil/.style={
		thick,
		in=180,
		out=0,
		shorten <=2pt,
		shorten >=1pt
	}
	]
	\node[block] (original) at (4.5,1.5) {Problem \eqref{eq:SfM}};
	\node[block,right=1.3cm of original] (less) {Problem \eqref{eq:lessConstrained}};
	\node[block,right=1.9cm of less] (expected) {Problem \eqref{eq:rankProblem}};
	\node[right=2.8cm of expected] (dummy) {};
	\node[block, right=1.7cm of expected] (B) {Problem \eqref{eq:pcaProblem}};
	\node[above=0.cm of B] (A) {Enforce $\mathbf{G}\succeq 0$};
	\node[block, right=1.7cm of B] (C) {Theorem \ref{thm:mainRes}};
	\path (original) edge[pil] node[thick,above] {Relax} (less);
	\path (less) edge[pil] node[thick,above] {Lemma \ref{eq:transLemma}} (expected);
	\path (expected.east) edge[npil] (A.west);
	\path (expected.east) edge[pil] node[thick,below] {Lemma \ref{lem:optimizedG}} (B.west);
	\path (A.east) edge[pil] (C.west);
	\path (B.east) edge[pil] (C.west);
	\end{tikzpicture}
	\caption{Flowchart of the proof of Theorem \ref{thm:mainRes}.}
\end{figure*}
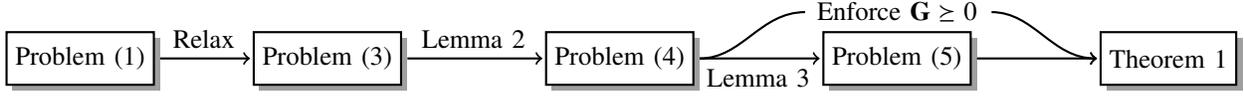

\begin{proof} of Corollary \ref{thm:mainCorollary}.
	
	We prove by verifying the three conditions in Theorem \ref{thm:mainRes}: (a) $\mathbf{\Sigma}_{k-1}\mathbf{V}_{k-1}^\top=\mathbf{S}\mathbf{G}^\top$; 
	(b) $\mathbf{G}\mathbb{1}_k = \mathbb{1}_n$;
	(c) $\mathbf{G}\succeq 0$.
	
	For (a), using the definition of $\mathbf{F}$, we can see $\mathbf{S} = r\sqrt{k(k-1)}\mathbf{B}^\top$.
	Thus, 
	\[
	\begin{aligned}
	\mathbf{SG}^\top =& r\sqrt{k(k-1)}\mathbf{B}^\top \left(\frac{\mathbf{B}\mathbf{U}_{k-1}^\top\mathbf{X}}{r\sqrt{k(k-1)}}+\frac{1}{k}\mathbb{1}_k\mathbb{1}_n^\top\right) \\
	=& \mathbf{B}^\top\mathbf{B}\mathbf{U}_{k-1}^\top\mathbf{X}+\frac{r\sqrt{k(k-1)}}{k}\mathbf{B}^\top\mathbb{1}_k\mathbb{1}_n^\top \\
	=& \mathbf{U}_{k-1}^\top\mathbf{X} = \mathbf{\Sigma}_{k-1}\mathbf{V}_{k-1}^\top,
	\end{aligned}
	\]
	where the last equality holds since $\mathbf{B}^\top\mathbf{B}=\mathbb{I}_{(k-1)}$ and $\mathbb{1}_k^\top\mathbf{B}=0$ since $\mathbf{H}_k \mathbb{1}_k=0$.

	For (b), a direct calculation gives
	\[
	\mathbf{G}\mathbb{1}_k = \frac{\mathbf{X}^\top \mathbf{U}_{k-1} \mathbf{B}^\top \mathbb{1}_k}{r\sqrt{k(k-1)}}+\frac{1}{k}\mathbb{1}_n\mathbb{1}_k^\top\mathbb{1}_k = \mathbb{1}_n,
	\]
	where we use the fact $\mathbf{G}\mathbb{1}_k = \mathbb{1}_n$.
	
	For (c), let $\mathbf{w}_i$ be the $i$th column of $\mathbf{\Sigma}_{k-1}\mathbf{V}_{k-1}^\top$.
	Note that $r \geq \| \mathbf{\Sigma}_{k-1}\mathbf{V}_{k-1}^\top \|_{2,\infty}$ by definition.
	Thus $\left\| \frac{\mathbf{B} \mathbf{w}_i}{r}\right\| = \left\| \frac{\mathbf{w}_i }{r} \right\| \leq 1$.
	Meanwhile, we note that $\mathbb{1}_k^\top \mathbf{B}\mathbf{w}_i = 0$. 
	Let $\mathbf{x} \in \mathbb{R}^k$ be $\frac{\mathbf{B}\mathbf{w}_i}{r}$.
	We have $\mathbf{x}^\top\mathbb{1}_k = 0$ and $\|\mathbf{x}\|_2\leq 1$.
	Using Lemma \ref{lem:numerical}, we conclude that $\left\|\frac{\mathbf{B}\mathbf{w}_i}{r}\right\|_\infty \leq \frac{\sqrt{k(k-1)}}{k}$.
	Then we have $\frac{\mathbf{B}\mathbf{w}_i}{r\sqrt{k(k-1)}} + \frac{1}{k}\mathbb{1}_k \succeq 0$,
	which directly indicates 
	$\mathbf{G}^\top = \frac{\mathbf{B}\mathbf{\Sigma}_{k-1}\mathbf{V}_{k-1}^\top}{r\sqrt{k(k-1)}} + \frac{1}{k}\mathbb{1}_k\mathbb{1}_n^\top = \frac{\mathbf{B}\mathbf{W}}{r\sqrt{k(k-1)}} + \frac{1}{k}\mathbb{1}_k\mathbb{1}_n^\top \succeq 0$.
	This completes the proof by Theorem \ref{thm:mainRes}.
\end{proof}

\begin{remark}\label{remark:proof_sketch_unique}
	Corollary \ref{thm:mainCorollary} shows that the output $\mathbf{F}$ and $\mathbf{G}$ by Algorithm \ref{alg:global}
	are an optimal solution of Problem \eqref{eq:SfM}. But we note that the solution of Problem \eqref{eq:SfM}
	is not unique since for the matrix $\mathbf{B}\in\mathbb{R}^{k\times (k-1)}$ defined in Algorithm \ref{alg:global} as
	the $(k-1)$-truncated eigenvectors of centralization matrix $\mathbf{H}_k$ is not unique. Indeed, for any $\mathbf{B}$, 
	matrix $\mathbf{BR}$ is still $(k-1)$-truncated eigenvectors of centralization matrix $\mathbf{H}_k$, where $\mathbf{R}^\top\mathbf{R}=\mathbb{I}_{(k-1)}$.
	That reveals there is non-uniqueness in the solution set of the Problem \eqref{eq:SfM}.
	Indeed, we develop a lower bound (\Cref{sec:arbitray}) on the distance between oracle $\mathbf{G}^*$ and the solution of Problem \eqref{eq:SfM} in Frobenius norm.
\end{remark}

\begin{remark}\label{remark:intuition}
	Here we provide some intuition behind the Algorithm \ref{alg:global}.
	For the Soft $k$-Means problem in \eqref{eq:SfM}, we want to find a best fit $k$ simplex constructed by prototype $\mathbf{F}$.
	If we have identified the optimal affine span of $\mathbf{F}$, then we can simply enlarge the simplex without change the
	objective function value. Thus, we choose the smallest regular $k$ simplex which contain the sphere with $\|\mathbf{X}\|_{2,\infty}$ as
	radius. Such a choices gives the $\mathbf{F}$ in Algorithm \ref{alg:global}.
\end{remark}

It is interesting to consider the situation when we can get the objective function minimized to zero.
Formally, we make a notion of S$k$Mable for data $\mathbf{X}$.

\newenvironment{definition}[1]{\par\noindent\textbf{Definition 1} \it}{}
\begin{definition}((S$k$Mable).
	We say $\mathbf{X} \in \mathbb{R}^{d\times n}$ is S$k$Mable if
	$
	\exists \mathbf{F}\in\mathbb{R}^{d\times k}, \mathbf{G}\in\mathbb{R}^{n\times k}:\mathbf{G}\mathbb{1}_k=\mathbb{1}_n
	$
	such that
	$
	\mathbf{X}=\mathbf{FG}^\top.
	$
\end{definition}

S$k$Mable is closely related to the notion LSDable introduced in \cite{arora2013similarity}.
Indeed, if a matrix is LSDable, then it is S$k$Mable (Theorem \ref{thm:jmlr13prop}).
We will discuss the relation between them and introduce a more reasonable and more understandable notion TI-LSDable in \Cref{sec:leftsto}.

Following result gives the necessary and sufficient condition for a data matrix $\mathbf{X}$ to be S$k$Mable.

\begin{thm}[Condition of S$k$Mable]\label{thm:sfmable}
	For $\mathbf{X}\in\mathbb{R}^{d\times n}$, there exists $\mathbf{F}\in\mathbb{R}^{d\times k}$ and $\mathbf{G}\in\mathbb{R}^{k\times d}, \mathbf{G}\succeq 0,\mathbf{G}\mathbb{1}_k = \mathbb{1}_n$ such that $\mathbf{X}=\mathbf{FG}^\top$ if and only if rank$(\mathbf{X}\mathbf{H}_n) \leq k-1$.
\end{thm}
The condition shown above is importance since it connects the abstract definition S$k$Kable to an easily checkable condition rank$(\mathbf{XH}_n) \leq k-1$, which will show its usefulness when we analyze the connection between S$k$Mable and LSDable in \Cref{sec:leftsto}.

\subsection{Proof Sketch} \label{sec:proofSketch}
In this subsection, we lay out the proof sketch of Theorem \ref{thm:mainRes}.

The strategy is that we first ignore the nonnegative constraint on $\mathbf{G}$ to 
identify the optimal affine span of $\mathbf{F}$.
Then, we enforce the nonnegative constraint to hold by enlarge $\mathbf{F}$ in its affine span.
Specifically, we consider following less constrained problem:
\begin{equation}\label{eq:lessConstrained}
\min_{\mathbf{F},\mathbf{G}\mathbb{1}_k = \mathbb{1}_n} \| \mathbf{X} - \mathbf{FG}^\top\|_F^2.
\end{equation}
Note that there is a translation invariance property for above objective function with the constraint $\mathbf{G}\mathbb{1}_k=\mathbb{1}_n$, that is, for any $\mathbf{s} \in \mathbb{R}^d$, we have
\[
\|\mathbf{X} - \mathbf{F}\mathbf{G}^\top\|_F^2 = \|\mathbf{X} - \mathbf{s}\mathbb{1}_n^\top - (\mathbf{F}-\mathbf{s}\mathbb{1}_k^\top)\mathbf{G}^\top\|_F^2.
\]
This property inspire us to build a reformulation of Problem \eqref{eq:lessConstrained} by finding a proper translation $\mathbf{s}$ which makes $\text{rank}(\mathbf{F}) \leq k-1$. 
In detail, we have following lemma:
\begin{Lemma1}\label{eq:transLemma}
	Following equivalence can be shown using the translation invariance of Problem \eqref{eq:lessConstrained}:
	\[
	\min_{\mathbf{F},\mathbf{G}\mathbb{1}_k = \mathbb{1}_n} \| \mathbf{X} - \mathbf{FG}^\top\|_F^2
	\Leftrightarrow
	\min_{\substack{\mathbf{F},\mathbf{G}\mathbb{1}_k = \mathbb{1}_n \\ \mathbf{s}, \text{rank}(\mathbf{F}) \leq k-1 }} \| \mathbf{X} - \mathbf{s}\mathbb{1}_n^\top- \mathbf{FG}^\top\|_F^2.
	\]
\end{Lemma1}
Thus, we are now interested in solving
\begin{equation}\label{eq:rankProblem}
\min_{\substack{\mathbf{F},\mathbf{G}\mathbb{1}_k = \mathbb{1}_n \\ \mathbf{s}, \text{rank}(\mathbf{F}) \leq k-1 }} \| \mathbf{X} - \mathbf{s}\mathbb{1}_n^\top- \mathbf{FG}^\top\|_F^2.
\end{equation}
The rank deficiency constraint of $\mathbf{F}$ gives us an opportunity to optimize $\mathbf{G}$ first. Specifically, we have following lemma:
\begin{Lemma1}\label{lem:optimizedG}
	Let $\mathbf{F}\in\mathbb{R}^{d\times k}$ be the prototype matrix satisfying the constraint $\text{rank}(\mathbf{F}) \leq k-1$ and has the thin SVD $\mathbf{F}=\mathbf{U}\mathbf{\Sigma}\mathbf{V}^\top$, where $\mathbf{\Sigma}\in\mathbb{R}^{(k-1)\times(k-1)}$. Let $ \mathbf{\Phi} = \mathbf{V}\mathbf{\Sigma}^{-1}\mathbf{U}^\top \left(\mathbf{X}-\mathbf{s}\mathbb{1}_n^\top\right)$. Define
	\[
	\mathbf{G}_*^\top = \mathbf{\Phi} 
	+
	\frac{\mathbf{v}_\bot\left(\mathbb{1}_n^\top - \mathbb{1}_k^\top \mathbf{\Phi} \right) }{\mathbb{1}_k^\top\mathbf{v}_\bot },
	\]
	where $\mathbf{v}_\bot \in \mathbb{R}^k$ satisfies $\|\mathbf{v}_\bot\|_2 = 1$ and $\mathbf{V}^\top \mathbf{v}_\bot = 0$.
	Then
	\[
	\mathbf{G}_* \in \mathop{\arg\min}_{\mathbf{G}\mathbb{1}_k = \mathbb{1}_n} \| \mathbf{X} - \mathbf{s}\mathbb{1}_n^\top- \mathbf{FG}^\top\|_F^2.
	\]
\end{Lemma1}
Plug the $\mathbf{G}_*$ in Lemma \ref{lem:optimizedG} into the objective function of Problem \eqref{eq:rankProblem} and use the easily checkable fact $\mathbf{F}\mathbf{\Phi} = \mathbf{U}\mathbf{U}^\top \left(\mathbf{X}-\mathbf{s}\mathbb{1}_n^\top\right), \mathbf{F}\mathbf{v}_\bot = 0$. We have
\[
\begin{aligned}
&\left\| \mathbf{X} - \mathbf{s}\mathbb{1}_n^\top- \mathbf{FG}^\top_*\right\|_F^2 \\
=& \left\| \mathbf{X} - \mathbf{s}\mathbb{1}_n^\top- \mathbf{F}\left(\mathbf{\Phi} 
+
\frac{\mathbf{v}_\bot\left(\mathbb{1}_n^\top - \mathbb{1}_k^\top \mathbf{\Phi} \right) }{\mathbb{1}_k^\top \mathbf{v}_\bot}\right) \right\|_F^2 \\
=& \left\| \mathbf{X} - \mathbf{s}\mathbb{1}_n^\top- \mathbf{U}\mathbf{U}^\top \left(\mathbf{X}-\mathbf{s}\mathbb{1}_n^\top\right) \right\|_F^2.
\end{aligned}
\]
Then, the Problem \eqref{eq:rankProblem} can be reformulated as
\begin{equation}\label{eq:pcaProblem}
\min_{\mathbf{s},\mathbf{U}^\top\mathbf{U}=\mathbb{I}_{(k-1)}} \left\| \left(\mathbb{I}_{d} - \mathbf{U}\mathbf{U}^\top\right)\left( \mathbf{X} - \mathbf{s}\mathbb{1}_n^\top \right) \right\|_F^2,
\end{equation}
where we omit the variables ($\mathbf{\Sigma}$ and $\mathbf{V}$) since they cannot effect the objective function.
Interestingly, the Problem \eqref{eq:pcaProblem} is the objective of the classical principal components analysis in the regression form. Thus the optimal $\mathbf{s}_* = \frac{1}{n}\mathbf{X}\mathbb{1}_n$ and $\mathbf{U}_*$ is the leading $(k-1)$ left singular vectors of matrix $\left(\mathbf{X} - \mathbf{s}_*\mathbb{1}_n^\top\right) = \mathbf{X}\mathbf{H}_n$. 
Now we have shown $\mathbf{F}=\mathbf{U}^*\mathbf{\Sigma}\mathbf{V}^\top$ and $\mathbf{G}=\mathbf{V}\mathbf{\Sigma}^{-1}\mathbf{U}^{*\top }
+
\frac{\mathbf{v}_\bot\left(\mathbb{1}_n^\top - \mathbb{1}_k^\top \mathbf{V}\mathbf{\Sigma}^{-1}\mathbf{U}^{*\top } \right) }{\mathbb{1}_k^\top\mathbf{v}_\bot }$ are a solution for the less constrained Problem \eqref{eq:SfM}.
But there are still unused degree of freedom in $\mathbf{\Sigma}$ and $\mathbf{V}$. Next, we will make use of these degree of freedom to enforce the constraint $\mathbf{G}\succeq 0$ satisfied.

For ease of notation, let $\overline{\mathbf{X}}=\mathbf{X} - \frac{1}{n}\mathbf{X}\mathbb{1}_n\mathbb{1}_n^\top = \mathbf{X}-\mathbf{s}^*\mathbb{1}_n^\top$ be centralized $\mathbf{X}$ and $\mathbf{\Sigma V}^\top$ be $\mathbf{S}$. We have
\[
\begin{aligned}
&\left\| \mathbf{X} - \mathbf{s}^*\mathbb{1}_n^\top - \mathbf{F}\mathbf{G}^\top \right\|_F^2 \\
=& \left\| \overline{\mathbf{X}} - \mathbf{U}^*\mathbf{S}\mathbf{G}^\top \right\|_F^2 \\
=& \left\| \mathbf{U}^*\mathbf{U}^{*\top} \overline{\mathbf{X}} - \mathbf{U}^*\mathbf{S}\mathbf{G}^\top + \left(\mathbb{I}_{d}-\mathbf{U}^*\mathbf{U}^{*\top}\right) \overline{\mathbf{X}} \right\|_F^2 \\
=& \left\| \mathbf{U}^*\mathbf{U}^{*\top} \overline{\mathbf{X}} - \mathbf{U}^*\mathbf{S}\mathbf{G}^\top \right\|_F^2 + \left\| \left(\mathbb{I}_{d}-\mathbf{U}^*\mathbf{U}^{*\top}\right) \overline{\mathbf{X}} \right\|_F^2,
\end{aligned}
\]
where the last equality uses the easily checkable fact 
\[
\left\langle  \mathbf{U}^*(\mathbf{U}^{*\top} \overline{\mathbf{X}} - \mathbf{S}\mathbf{G}^\top), (\mathbb{I}_{d}-\mathbf{U}^*\mathbf{U}^{*\top}) \overline{\mathbf{X}} \right\rangle = 0.
\]
Note that $ \left\| \left(\mathbb{I}_{d}-\mathbf{U}^*\mathbf{U}^{*\top}\right) \overline{\mathbf{X}} \right\|_F^2$ is constant with respect to $\mathbf{S}$ and $\left\| \mathbf{U}^*\mathbf{U}^{*\top} \overline{\mathbf{X}} - \mathbf{U}^*\mathbf{S}\mathbf{G}^\top \right\|_F^2 = \left\|\mathbf{U}^{*\top} \overline{\mathbf{X}} - \mathbf{S}\mathbf{G}^\top \right\|_F^2$.
Thus, we can solve $\mathbf{S}$ and $\mathbf{G}$ by solving following problem
\begin{equation}\label{eq:add_constraint_G}
\min_{\mathbf{S},\mathbf{G}\mathbb{1}_k = \mathbb{1}_n} \left\|\mathbf{U}_{k-1}^\top \overline{\mathbf{X}} - \mathbf{S}\mathbf{G}^\top \right\|_F^2.
\end{equation}
Indeed, the objective function in Problem \eqref{eq:add_constraint_G} can be minimized to zero with $\mathbf{G}$ has the form mentioned in Lemma 
\ref{lem:optimizedG}.
That is the reason why $\mathbf{\Sigma V}^\top$ disappeared in Problem \eqref{eq:pcaProblem}, since in that case $\left\| \mathbf{X} - \mathbf{s}^*\mathbb{1}_n^\top - \mathbf{F}\mathbf{G}^\top \right\|_F^2=\left\| \left(\mathbb{I}_{d}-\mathbf{U}^*\mathbf{U}^{*\top}\right) \overline{\mathbf{X}} \right\|_F^2$.

Now we back to Theorem \ref{thm:mainRes}.
Note that if there exists $\mathbf{S}\in\mathbb{R}^{(k-1)\times k}$ and $\mathbf{G}$ satisfy (a) $\mathbf{\Sigma}_{k-1}\mathbf{V}_{k-1}^\top = \mathbf{S}\mathbf{G}^\top$; (b) $\mathbf{G}\mathbb{1}_k = \mathbb{1}_n$; (c) $\mathbf{G} \succeq 0$,
then the objective function in Problem \eqref{eq:add_constraint_G} can still be minimized to zero.
Combining with $\mathbf{U}^*$ minimizing $\left\| \left(\mathbb{I}_{d}-\mathbf{U}^*\mathbf{U}^{*\top}\right) \overline{\mathbf{X}} \right\|_F^2$, that completes the proof of Theorem \ref{thm:mainRes}.

\section{Discussions} \label{sec:discuss}
In this section, we discuss the solution non-uniqueness issue and the stability of the soft $k$-Means modal (\Cref{sec:arbitray,sec:stability}) by providing lower and upper bound. Then a new heuristic model, named Minimal Volume Soft $k$-Means, is proposed in \Cref{sec:MVSkM} with an optimization algorithm, which is guaranteed to be descent, to address the non-uniqueness issue. Finally, we proposed the notion of Translation Invariant LSDable towards solving an open problem mentioned in \cite{arora2013similarity}.

\subsection{Stability}\label{sec:stability}
As a byproduct of  \Cref{sec:proofSketch},
we can analyze the stability of the soft $k$-Means problem 
by consider the perturbation in data matrix.
With \Cref{sec:arbitray}, we know that bound the optimal variables for perturbed data is hopeless.
We provide a stability upper bound for the optimal objective value as follows.
\begin{thm}[Stability]\label{thm:stability}
	Assume additively perturbed data $\widetilde{\mathbf{X}} = \mathbf{X} + \mathbf{E}$. 
	We have
	\[
	\|\mathbf{X} - \widetilde{\mathbf{F}}_*\widetilde{\mathbf{G}}_*^\top \|_F^2 \leq  2\|\mathbf{E}\|_F^2 + \|\mathbf{X} - \mathbf{F}_*\mathbf{G}_*^\top \|_F^2.
	\]
\end{thm}

\begin{remark}
	Theorem \ref{thm:sfmable} shows that the objective function value of Soft $k$-Means is stable and up to a constant factor of the oracle
	difference $\|\mathbf{E}\|_F^2$. This is definitely a good news.
	But for the variables, that is prototype $\mathbf{F}$ and membership $\mathbf{G}$, we cannot expect such a upper bound.
	Indeed, even for unperturbed data, we can prove a lower bound on the distance between a solution of Problem \eqref{thm:mainRes} and the oracle membership $\mathbf{G}^*$.
\end{remark}

\subsection{Solutions Non-uniqueness}\label{sec:arbitray}

In this section, we report a bad news on the soft $k$-Means model, that is,
there is substantial non-uniqueness in its optima. 
Formally, we provide following lower bound on the distance between its optima and oracle 
member matrix $\mathbf{G}^*$ in Frobenius norm.
\begin{thm}[Lower bound]
	Let the best $(k-1)$-rank approximation of $\mathbf{X}$ be $\mathbf{X}_{k-1}$.
	For any oracle $\mathbf{G}^*$, there exists $\mathbf{G}$, which is a solution of Problem \eqref{eq:SfM},  such that 
	\[
	\|\mathbf{G} - \mathbf{G}^*\|_F^2 \geq \frac{1}{r\sqrt{k(k-1)}} \| \mathbf{X}_{k-1}\|_F^2.
	\]
\end{thm}

\begin{remark}
	The solutions non-uniqueness issue is mainly caused since
	the volume\footnote{We abuse the notion volume here to refer to the volume relative to the affine subspace.} of the polyhedron of $\mathbf{F}$ can be made infinitely
	large to minimize the Problem \eqref{eq:add_constraint_G} in the affine span of $\mathbf{F}$,.
	We note that this issue has been realized in \cite{nascimento2003modeling} by arguing
	that for any $\alpha > 0$, $\alpha\mathbf{f}_i \cdot \frac{1}{\alpha} g_{ij}$ remain unchanged.
	Thus $\|\mathbf{f}_i\|$ can tend to infinity without change the objective value.
	In the next subsection, we try to address this issue by jointly minimizing the objective function of Soft $k$-Means
	and the volume of the $\mathbf{F}$ polyhedron.
\end{remark}

\subsection{The Minimal Volume Soft $k$-Means}\label{sec:MVSkM}
In this subsection, we report a new model which minimizes the objective function of soft $k$-Means while keep the volume of polyhedron
of prototypes $\mathbf{F}$ minimized. While such a problem has been proved to be NP-hard to solve \cite{zhou2002algorithms}, 
we propose an algorithm that is guaranteed to make objective function value descent. For simplicity, in this subsection, we assume the data is centered, that is, $\mathbf{X}\mathbb{1}_n = 0$.

Since the volume of a polyhedron with affine dimension less than the dimension ambient is always $0$, we consider the volume in the affine space rather than in the ambient space, that is, in $\text{span}(\mathbf{U})$ rather than $\mathbb{R}^d$. 

The volume of polyhedron generated by column vector of $\mathbf{F}$ in subspace $\text{span}(\mathbf{U})$ can be written as
\[
\text{vol}_{\text{span}(\mathbf{U})}(\mathbf{F}) = \frac{1}{(k-1)!} \det 
\left[ \begin{array}{c}
\mathbf{\Sigma V}^\top \\ \mathbb{1}_k^\top \end{array} 
\right]
\]
Without loss of generality, we minimize it within $\log$ and absorb the $1/(k-1)!$ into a hyper-parameter $\lambda$. The problem we concern is
\[
\min_{\mathbf{F},\mathbf{G}\succeq 0,\mathbf{G}\mathbb{1}_k = \mathbb{1}_n} \| \mathbf{X} - \mathbf{FG}^\top\|_F^2 + \lambda \log \det 
\left[ \begin{array}{c}
\mathbf{\Sigma V}^\top \\ \mathbb{1}_k^\top \end{array}
\right].
\]
To solve above problem, we make following observation
\begin{Lemma1}
	Assume $\mathbf{F}=\mathbf{U\Sigma V}^\top$ such that $\mathbf{F}\mathbb{1}_k = 0$. Then,
	\[
	\log \det 
	\left[ \begin{array}{c}
	\mathbf{\Sigma V}^\top \\ \mathbb{1}_k^\top \end{array}
	\right] = \log(\sqrt{k}) +  \sum_{i=1}^{k-1} \log(\sigma_i(\mathbf{F})).
	\]
\end{Lemma1}
Using $\mathbf{X}\mathbb{1}_n = 0$, we can see from \Cref{sec:proofSketch} that the condition $\mathbf{F}\mathbb{1}_k = 0$ satisfied,
which leads to a possible interesting problem
\[
\min_{\mathbf{F},\mathbf{G}\succeq 0,\mathbf{G}\mathbb{1}_k = \mathbb{1}_n} \| \mathbf{X} - \mathbf{FG}^\top\|_F^2 + \frac{\lambda}{2} \sum_{i=1}^k \log(\sigma_i^2(\mathbf{F})).
\]
But sadly, above problem approach $-\infty$ when the singular value of $\mathbf{F}$ tends to zero, which is a trivial solution due to the lower unboundedness of logarithm.
Thus, we solve following problem with an $\epsilon > 0$ modification:
\begin{equation}\label{eq:mvsfm}
\min_{\mathbf{F},\mathbf{G}\succeq 0,\mathbf{G}\mathbb{1}_k = \mathbb{1}_n} \| \mathbf{X} - \mathbf{FG}^\top\|_F^2 + \frac{\lambda}{2} \sum_{i=1}^k \log(\sigma_i^2(\mathbf{F}) + \varepsilon).
\end{equation}
Inspired by recent research on multi-task learning \cite{nie2018calibrated}, we design an iteratively re-weighted strategy  by performing alternating minimization on $\mathbf{F}$ and $\mathbf{G}$. Specifically, we update $\mathbf{F}_{t+1}$ by solving 
\[
\mathbf{F}_{t+1} = \mathop{\arg\min}_{F} \| \mathbf{X} - \mathbf{F}\mathbf{G}_t^\top \|^2_F + \lambda \text{Tr}\left(\mathbf{D}_t\mathbf{F}^\top\mathbf{F}\right),
\]
where $\mathbf{D} = \mathbf{V}_t \text{diag}\left( \{ 1/(\sigma_i^2(\mathbf{F}_t)+\varepsilon)  \}_{i=1}^{k} \right)$ and $\mathbf{F}_t = \mathbf{U}_t\mathbf{\Sigma}_t\mathbf{V}^\top$. Above subproblem for $\mathbf{F}$ can be solve easily using first-order optimal condition. Then, we update $\mathbf{G}_{t+1}$, by solve following $n$ problem
\[
(\mathbf{g}_i)_t = \mathop{\arg\min}_{\mathbf{g}\mathbb{1}_k = 1, \mathbf{g} \succeq 0} \| \mathbf{x}_i - \mathbf{F}_{t+1}\mathbf{g}^\top \|_2^2,
\]
which is convex and can be solved with Projected Gradient Descent or even with Nesterov's acceleration.

We summarize the procedure to solve the MVS$k$M model in following algorithm.

\begin{algorithm}[H]  
	\caption{Solve the MVS$k$M Problem}  
	\label{alg:MVSfM}
	\begin{algorithmic}[1]  
		\Require $\mathbf{X}\in\mathbb{R}^{d\times n}$, $\lambda$, and the clusters number $k$
		\State $\mathbf{F}_0 \leftarrow$ random initialized prototypes.
		\Repeat
		\State $\mathbf{V}_t \leftarrow$ left SV of $\mathbf{F}_t$
		\State $\mathbf{D}_t \leftarrow\mathbf{V}_t \text{diag}(\{ 1/(\sigma_i^2(\mathbf{F}_t) + \varepsilon) \}_{i=1}^k) \mathbf{V}_t^\top$
		\State $\mathbf{F}_{t+1} \leftarrow \mathbf{X}\mathbf{G}_t(\mathbf{G}_t^\top\mathbf{G}_t + \lambda \mathbf{D}_t)^{-1}$
		\State $\displaystyle \mathbf{G}_{t+1} \leftarrow \mathop{\arg\min}_{\mathbf{G}\succeq 0,\mathbf{G}\mathbb{1}_k = \mathbb{1}_n} \| \mathbf{X} - \mathbf{F}_{t+1}\mathbf{G}^\top\|_F^2.$
		\Until{converge}
		
		\Ensure Prototypes $\mathbf{F} \in \mathbb{R}^{d\times c}$, membership $\mathbf{G} \in \mathbb{R}^{c \times n}$.
	\end{algorithmic}  
\end{algorithm} 

Here, we provide some analysis for above algorithm. For ease of notation, we shall write the objective function as
\[
\mathcal{L}(\mathbf{F}, \mathbf{G})=\| \mathbf{X} - \mathbf{F}\mathbf{G}^\top\|_F^2 + \frac{\lambda}{2} \sum_{i=1}^k \log(\sigma_i^2(\mathbf{F}) + \varepsilon).
\]
Following result shows that our algorithm to solving MVS$k$M is a descent method.
\begin{thm}[Descent]\label{thm:mvsfmDescent}
	Algorithm \ref{alg:MVSfM} is a descent method, that is,
	\[
	\mathcal{L}(\mathbf{F}_{t+1}, \mathbf{G}_{t+1}) \leq \mathcal{L}(\mathbf{F}_{t}, \mathbf{G}_{t}).
	\]
\end{thm}
Since the objective function (after $\varepsilon$ modification) of Problem \eqref{eq:mvsfm} is bounded from below ($\log(\sigma_i^2(\mathbf{F}) + \varepsilon) \geq \log(\varepsilon)$), the iterative scheme will finally converge from Theorem \ref{thm:mvsfmDescent}.
%

\begin{figure*}[!htbp]\label{fig:toy}
	\centering
	\includegraphics[width=1\textwidth]{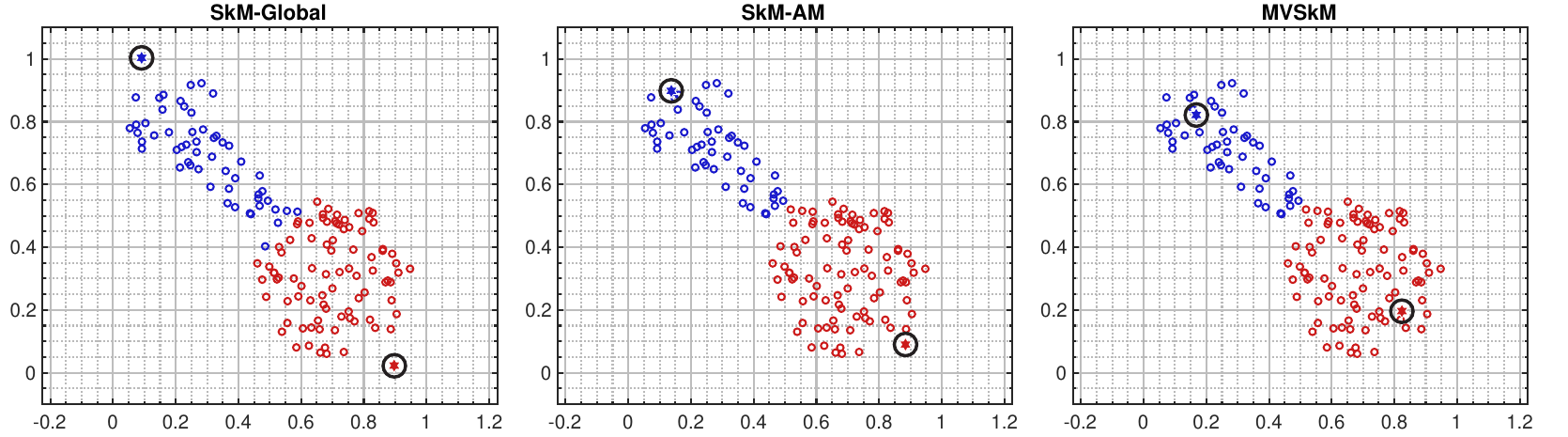}
	\caption{Comparison of solving the Soft $k$-Means problem with the proposed global optimal algorithm (left), alternating minimization (middle), and the Minimal Volume Soft $k$-Means model on synthetic dataset. The synthetic data set is the \texttt{fcmdata.dat} along with MATLAB.}
\end{figure*}

\subsection{On Left Stochastic Clustering}\label{sec:leftsto}
First, we introduce the notion of LSDable appeared in \cite{arora2013similarity}.
\newenvironment{definition2}[1]{\par\noindent\textbf{Definition 2} \it}{}
\begin{definition2}[(LSDable)
	Similarity matrix $\mathbf{K} \in \mathbb{R}^{n\times n}$ is $k$-LSDable if $\exists c \in \mathbb{R}, \mathbf{G} \in \mathbb{R}^{n\times k}: \mathbf{G}\succeq 0,\mathbf{G}\mathbb{1}_k = \mathbb{1}_n$ such that $\mathbf{K} = \frac{1}{c}\mathbf{GG}^\top$.
\end{definition2}

LSDable is closely related to the soft $k$-Means problem considered in this paper by following theorem.
\begin{thm}[\cite{arora2013similarity}]\label{thm:jmlr13prop}
	Assume $\mathbf{K}=\mathbf{X}^\top\mathbf{X}$ is $k$-LSDable. Then the optimal $\mathbf{G}$ in LSD problem is also an optimum in the soft $k$-Means problem of $\mathbf{X}$.
\end{thm}

Above theorem indicates that if $\mathbf{X}^\top\mathbf{X}$ is LSDable then it is S$k$Mable, which means that
\{LSDable\} $\subseteq$ \{S$k$Mable\}.
Ideally, the clustering property of given $\mathbf{X}$ should not change if we translate all points simultaneously with vector $\mathbf{s}\in\mathbb{R}^d$.
But for a LSDable matrix, say $\mathbf{X}^\top\mathbf{X}$, the simultaneous translation, then it is $\left(\mathbf{X}+\mathbf{s}\mathbb{1}_n^\top\right)^\top\left(\mathbf{X}+\mathbf{s}\mathbb{1}_n^\top\right)$, may make it not LSDable. To make the notion of LSDable translation invariant,
we extend it to following definition.

\newenvironment{definition3}[1]{\par\noindent\textbf{Definition 3} \it}{}
\begin{definition3}((Translation Invariant Set).
	Define 
	\[
	\mathbb{K}(\mathbf{X}) = \left\{ \left( \mathbf{X} + \mathbf{t}\mathbb{1}_n^\top  \right)^\top \left( \mathbf{X} + \mathbf{t}\mathbb{1}_n^\top \right) : \mathbf{t}\in\mathbb{R}^d \right\}.
	\]
\end{definition3}

\newenvironment{definition4}[1]{\par\noindent\textbf{Definition 4} \it}{}
\begin{definition4}((Translation Invariant LSDable).
	We say $\mathbf{K}=\mathbf{X}^\top\mathbf{X}$ is Translation Invariant LSDable (TI-LSDable) if there exists
	$
	\mathbf{K}^\prime \in \mathbb{K}(\mathbf{X})
	$ such that $\mathbf{K}^\prime$ is LSDable.
\end{definition4}

Indeed, one can easily see 
\[
\{\text{LSDable}\} \subseteq \{\text{TI-LSDable}\} \subseteq \{\text{S$k$Mable}\},
\]
where the first subset notation holds since that if $\mathbf{K}$ is LSDable, it must be TI-LSDable 
and the second one holds from the fact that if $\mathbf{X}$ is S$k$Mable then $\mathbf{X}+\mathbf{s}\mathbb{1}_n^\top$ is also S$k$Mable..

The problem we concern in this subsection is 
\begin{equation}\label{eq:TILSD}
\min_{\substack{c,\mathbf{G}\succeq 0, \mathbf{G}\mathbb{1}_k = \mathbb{1}_n \\ \mathbf{K}^\prime \in \mathbb{K}(\mathbf{X}) } }
\| \mathbf{K}^\prime - \frac{1}{c}\mathbf{G}^\top\mathbf{G} \|_F^2
\end{equation}
The main result in this subsection is closely related to an open problem proposed in \cite{arora2013similarity},
in which it asked for the relation between LSD of unLSDable $\mathbf{X}^\top\mathbf{X}$ and the solution of Soft $k$-Means for $\mathbf{X}$.
In the sequel, we provide such a relation but with TI-LSDable rather than LSDable.
\begin{thm}[Condition of TI-LSDable]\label{thm:TILSDable}
	$\mathbf{K}$ is TI-LSDable if and only if rank$(\mathbf{H}\mathbf{K}\mathbf{H}) \leq k-1$.
\end{thm}
Theorem \ref{thm:TILSDable} connects the abstract TI-LSDable notions with an easily checkable condition rank$(\mathbf{H}\mathbf{K}\mathbf{H}) \leq k-1$.
Using this connection, we have a simple corollary build relation between the TI-LSDable and S$k$Mable.
\begin{corollary}[SfMable is TI-LSDable]
	Assume $\mathbf{K}=\mathbf{X}^\top\mathbf{X}$. $\mathbf{K}$ is TI-LSDable if and only if $\mathbf{X}$ is S$k$Mable.
\end{corollary}
\begin{proof}
	Note that $\text{rank}(\mathbf{H}_n\mathbf{X}^\top\mathbf{XH}_n) = \text{rank}(\mathbf{XH}_n)$. For the if direction, since $\mathbf{X}$ is S$k$Mable, by Theorem \ref{thm:sfmable}, $\text{rank}(\mathbf{H}_n\mathbf{X}^\top\mathbf{XH}_n) = \text{rank}(\mathbf{XH}_n)\leq k-1$. For the only if direction, since $\mathbf{K}=\mathbf{X}^\top\mathbf{X}$ is TI-LSDable, by Theorem \ref{thm:TILSDable}, 
	$\text{rank}(\mathbf{XH}_n) = \text{rank}(\mathbf{H}_n\mathbf{X}^\top\mathbf{XH}_n) \leq k-1$
\end{proof}
\begin{thm}\label{thm:tilsdEqual}
	Assume kernel $\mathbf{K}\in\mathbb{R}^{n\times n}$ satisfies rank$(\mathbf{K}) \leq n-1$.
	If $\mathbf{G}^*$ is a solution of the TI-LSD in Problem \eqref{eq:TILSD} for kernel $\mathbf{K}$,
	then there exists $\mathbf{X}$ such that $\mathbf{K}=\mathbf{X}^\top\mathbf{X}$ and $\mathbf{G}^*$ is also a solution of S$k$M
	in Problem \eqref{eq:SfM}.
\end{thm}
\begin{remark}
	In Theorem \ref{thm:tilsdEqual}, we show that for a given kernel $\mathbf{K}=\mathbf{XX}^\top$ satisfying rank$(\mathbf{K})\leq n-1$,
	its the optimal  membership matrix for the LSD problem in Problem \eqref{eq:LSD} is still an optimal membership
	for the Soft $k$-Means problem on $\mathbf{X},$
	which gives an answer for the relation between TI-LSD of unLSDable matrix $\mathbf{X}^\top\mathbf{X}$
	and the solution of soft $k$-Means for $\mathbf{X}$.
\end{remark}

\begin{table*}[]
	\centering
	\caption{Comparison of S$k$M-Global, S$k$M-AM, and MVS$k$M on Real-world Datasets.}
	\label{tab:exps}
	\begin{tabular}{@{}cccccccccc@{}}
		\toprule
		& \multicolumn{3}{c}{MSRCv1}                                                                                         & \multicolumn{3}{c}{ORL Face}                                                                                       & \multicolumn{3}{c}{Numerical Numbers}                                                         \\ \cmidrule(l){2-10} 
		& \multicolumn{1}{c|}{ACC}             & \multicolumn{1}{c|}{NMI}             & \multicolumn{1}{c|}{Purity}          & \multicolumn{1}{c|}{ACC}             & \multicolumn{1}{c|}{NMI}             & \multicolumn{1}{c|}{Purity}          & \multicolumn{1}{c|}{ACC}             & \multicolumn{1}{c|}{NMI}             & Purity          \\ \midrule
		\multicolumn{1}{c|}{SkM-AM}     & \multicolumn{1}{c|}{0.5714}          & \multicolumn{1}{c|}{0.4429}          & \multicolumn{1}{c|}{0.5952}          & \multicolumn{1}{c|}{0.4100}          & \multicolumn{1}{c|}{0.6076}          & \multicolumn{1}{c|}{0.4200}          & \multicolumn{1}{c|}{0.4775}          & \multicolumn{1}{c|}{0.4475}          & 0.4780          \\
		\multicolumn{1}{c|}{SkM-Global} & \multicolumn{1}{c|}{0.5000}          & \multicolumn{1}{c|}{0.4334}          & \multicolumn{1}{c|}{0.5476}          & \multicolumn{1}{c|}{0.3150}          & \multicolumn{1}{c|}{0.5524}          & \multicolumn{1}{c|}{0.3425}          & \multicolumn{1}{c|}{0.4205}          & \multicolumn{1}{c|}{0.4119}          & 0.4435          \\
		\multicolumn{1}{c|}{MVSkM}      & \multicolumn{1}{c|}{\textbf{0.7476}} & \multicolumn{1}{c|}{\textbf{0.6495}} & \multicolumn{1}{c|}{\textbf{0.7476}} & \multicolumn{1}{c|}{\textbf{0.4850}} & \multicolumn{1}{c|}{\textbf{0.7071}} & \multicolumn{1}{c|}{\textbf{0.5050}} & \multicolumn{1}{c|}{\textbf{0.6515}} & \multicolumn{1}{c|}{\textbf{0.6006}} & \textbf{0.6550} \\ \bottomrule
	\end{tabular}
\end{table*}



\section{Experiments} \label{sec:exps}

In this section, we provide experimental results to back up our theoretical analysis.
In detail, we visualize the clustering results with a two dimensional synthetic dataset to demonstrate the difference between alternating minimization \cite{nascimento2016applying}, the proposed Algorithm \ref{alg:global}, and MVS$k$M.
Then we conduct experiments on three real-world datasets to perform
clustering performance comparison of solving Soft $k$-Means with alternating minimization (S$k$M-AM), the proposed Algorithm \ref{alg:global} (S$k$M-Global) and the MVS$k$M model.

\subsection{Synthetic Dataset}
To show the prototype selection and membership assignment difference between solving Soft $k$-Means with
alternating minimization (S$k$M-AM), with Algorithm \ref{alg:global} (S$k$M-Global) and the MVS$k$M model,
we visualize the clustering result on a synthetic dataset.

In detail, we plot the learned prototypes and visualize the membership assignment from the three comparison methods.
The synthetic data is the \texttt{fcmdata.mat} along with the MATLAB.

The clustering results visualization are shown in \Cref{fig:toy}.
We make following remarks:
\begin{itemize}
	\item It is observed that both S$k$M-AM and S$k$M-Global tend to learn the prototypes $\mathbf{F}$ outside of the data convex hull. The reason is that one can always enlarge the simplex with $\mathbf{F}$ as vertices without changing the objective function, which has been mentioned in Remark \ref{remark:intuition}.
	\item For the newly proposed model MVS$k$M, with proper trading off between the loss function and the volume regularization, 
	we can obtain prototypes insider of the data convex hull, which is more reasonable in most cases.
\end{itemize}
\subsection{Real-world Datasets}

To validate the clustering performance, we conduct experiments on three real-world datasets. They are MSRCv1 \cite{winn2005locus}, ORL Face \cite{samaria1994parameterisation}, and Numerical Numbers \cite{asuncion2007uci}.
For the fuzzy membership matrix, we discretize the membership matrix $\mathbf{G}$ by selecting the class with max membership grade for every data point.
For every dataset, we follow the preprocessing method in \cite{nie2018multiview} and use the commonly used clustering performance measures:
accuracy (ACC), Normalized Mutual Information (NMI), and Purity.

The results are summarized in \Cref{tab:exps} and the best results are marked in bold face. Following remarks can be made:
\begin{itemize}
	\item It is interesting to see S$k$M-Global cannot outperform S$k$M-AM, while it outputs the global solution
	of Problem \eqref{eq:SfM}. The reason is that the solution non-uniqueness issue causes performance degradation.
	\item With proper trading off between the loss function and the polyhedron volume, MVS$k$M outperforms both
	S$k$M-AM and S$k$M-Global, which validates the motivation of the MVS$k$M model.
\end{itemize}
\section{Conclusion}\label{sec:conclusion}

In this paper, we present a sufficient condition for a feasible solution of Soft $k$-Means problem to be globally optimal.
Then, we report an algorithm (Algorithm \ref{alg:global}) whose output satisfies this sufficient condition, thus globally solve the Soft $k$-Means problem.
Moreover, for the Soft $k$-Means problem, we provide interesting discussions on 
stability, solutions non-uniqueness, and connection with LSC.
Then, we proposed an new clustering model, named Minimal Volume Soft $k$-Means (MVS$k$M), to address the solutions non-uniqueness issue in Soft $k$-Means.
Finally, experimental results support our theoretical results.


\appendix
\section{Appendix}
\subsection{Theorems}
\newtheorem{thmm}{\bf Theorem}
\setcounter{thmm}{1}
\begin{thmm}[Condition of S$k$Mable]
	For $\mathbf{X}\in\mathbb{R}^{d\times n}$, there exists $\mathbf{F}\in\mathbb{R}^{d\times k}$ and $\mathbf{G}\in\mathbb{R}^{k\times d}, \mathbf{G}\succeq 0,\mathbf{G}\mathbb{1}_k = \mathbb{1}_n$ such that $\mathbf{X}=\mathbf{FG}^\top$ if and only if rank$(\mathbf{X}\mathbf{H}_n) \leq k-1$.
\end{thmm}
\begin{proof}
	For the {\em if} direction, choose $\textbf{F}$ and $\textbf{G}$ according to the Algorithm \ref{alg:global}. 
	Directly calculation gives $\textbf{FG}^\top = \textbf{X}$.
	
	For the {\em only if} direction, note that
	\[
	\textbf{X} - \frac{1}{k}\textbf{F}\mathbb{1}_k\mathbb{1}_n^\top = \textbf{F}\left(\textbf{G}^\top - \frac{1}{k}\textbf{F}\mathbb{1}_k\mathbb{1}_k^\top\textbf{G}^\top\right) = \textbf{FH}_k\textbf{G}.
	\]
	
	Using Lemma \ref{lem:rankXH}, we have
	\[
	\begin{aligned}
	\text{rank}(\textbf{XH}_n) &\leq \text{rank}\left(\textbf{X} - \frac{1}{k}\textbf{F}\mathbb{1}_k\mathbb{1}_n^\top\right)\\
	&= \text{rank}\left(\textbf{F}\textbf{H}_k\textbf{G}^\top\right)\\
	&\leq \text{rank}(\textbf{H}_k) \\
	&= k-1.
	\end{aligned}
	\]
\end{proof}
\begin{thmm}[Stability]
	Assume additively perturbed data $\widetilde{\mathbf{X}} = \mathbf{X} + \mathbf{E}$. 
	We have
	\[
	\|\mathbf{X} - \widetilde{\mathbf{F}}_*\widetilde{\mathbf{G}}_*^\top \|_F^2 \leq  2\|\mathbf{E}\|_F^2 + \|\mathbf{X} - \mathbf{F}_*\mathbf{G}_*^\top \|_F^2.
	\]
\end{thmm}
\begin{proof}
	Let $r_1 = $rank$(\textbf{X})$, $r_2 = $rank$(\widetilde{\textbf{X}})$, and $r_3 = $rank$(\textbf{E})$. Then $r = \max\{r_1, r_2, r_3\}.$
	\[
	\begin{aligned}
	\|\textbf{X} - \widetilde{\textbf{F}}_*\widetilde{\textbf{G}}_*^\top \|_F^2 =& \|\textbf{X} - \widetilde{\textbf{X}} + \widetilde{\textbf{X}} - \widetilde{\textbf{F}}_*\widetilde{\textbf{G}}_*^\top \|_F^2 \\
	\leq& \|\textbf{X} - \widetilde{\textbf{X}} \|_F^2 + \| \widetilde{\textbf{X}} - \widetilde{\textbf{F}}_*\widetilde{\textbf{G}}_*^\top \|_F^2 \\
	\leq& \|\textbf{E}\|_F^2 + \|\textbf{X} - \textbf{F}_*\textbf{G}_*^\top \|_F^2 + \sqrt{\sum_{i=k}^{r_2}\tilde{\sigma}_i^2} - \sqrt{\sum_{i=k}^{r_1}\sigma_i^2} \\
	\overset{(a)}{\leq} & \|\textbf{E}\|_F^2 + \|\textbf{X} - \textbf{F}_*\textbf{G}_*^\top \|_F^2 + \sqrt{\sum_{i=k}^{r}\left(\tilde{\sigma}_i - \sigma_i\right)^2 } \\
	\overset{(b)}{\leq}&   2\|\textbf{E}\|_F^2 + \|\textbf{X} - \textbf{F}_*\textbf{G}_*^\top \|_F^2,
	\end{aligned}
	\]
	where $(a)$ holds since $\|a\|-\|b\|\leq\|a-b\|$ and $(b)$ is from Mirsky's inequality (Lemma \ref{lem:mirsky}).
\end{proof}
\begin{thmm}[Lower bound]
	Let the best $(k-1)$-rank approximation of $\mathbf{X}$ be $\mathbf{X}_{k-1}$.
	For any oracle $\mathbf{G}^*$, there exists $\mathbf{G}$, which is a solution of Problem \eqref{eq:SfM},  such that 
	\[
	\|\mathbf{G} - \mathbf{G}^*\|_F^2 \geq \frac{1}{r\sqrt{k(k-1)}} \| \mathbf{X}_{k-1}\|_F^2.
	\]
\end{thmm}
\begin{proof}
	Let
	\[
	\begin{aligned}
	\mathbb{A} &= \arg_{\mathbf{G}}\min_{\mathbf{F},\mathbf{G}\succeq 0,\mathbf{G}\mathbb{1}_k = \mathbb{1}_n} \| \mathbf{X} - \mathbf{FG}^\top\|_F^2 \\
	\mathbb{B} &= \left\{ \mathbf{G} \left|  \mathbf{G}=\frac{1}{r\sqrt{k(k-1)}}\mathbf{B}\mathbf{U}_{k-1}^\top\mathbf{X}+\frac{1}{k}\mathbb{1}_k\mathbb{1}_n^\top \right.\right\}.
	\end{aligned}
	\]
	From Theorem \ref{thm:mainRes}, we know $\mathbb{A}\supseteq \mathbb{B}$. Indeed, for any $\mathbf{G}_1 \in \mathbb{B}$, exists $\mathbf{G}_2 \in \mathbb{B}$ such that
	\[
	\begin{aligned}
	\|\mathbf{G}_1 - \mathbf{G}_2\|_F^2=&\left\| \frac{1}{r\sqrt{k(k-1)}}\mathbf{B}\mathbf{R}\mathbf{U}_{k-1}^\top\mathbf{X}
	-\frac{1}{r\sqrt{k(k-1)}}\mathbf{B}\mathbf{U}_{k-1}^\top\mathbf{X} \right\|_F^2 \\
	=&\frac{1}{r\sqrt{k(k-1)}} \left\| \mathbf{B}(\mathbf{R}-\mathbb{I})\mathbf{U}_{k-1}^\top\mathbf{X}\right\|_F^2 \\
	=&\frac{1}{r\sqrt{k(k-1)}} \left\|(\mathbf{R}-\mathbb{I})\mathbf{U}_{k-1}^\top\mathbf{X}\right\|_F^2 \\
	\geq& \frac{2}{r\sqrt{k(k-1)}} \left\|\mathbf{U}_{k-1}^\top\mathbf{X}\right\|_F^2 \\
	=& \frac{2}{r\sqrt{k(k-1)}} \left\|\mathbf{X}_{k-1}\right\|_F^2 
	\end{aligned}
	\]
	Note that
	\[
	\begin{aligned}
	\max_{\mathbf{G}\in\mathbb{A}} \| \mathbf{G} - \mathbf{G}^* \|_F^2 
	\geq & \max_{\mathbf{G}\in\mathbb{B}} \| \mathbf{G} - \mathbf{G}^* \|_F^2 \\
	\geq & \frac{1}{2}\left(\|\mathbf{G}_1 - \mathbf{G}^*\|_F + \|\mathbf{G}_2 - \mathbf{G}^*\|_F^2 \right) \\
	\geq & \frac{1}{2} \|\mathbf{G}_1 - \mathbf{G}_2\|_F^2 \\ 
	\geq & \frac{1}{r\sqrt{k(k-1)}} \left\|\mathbf{X}_{k-1}\right\|_F^2,
	\end{aligned}
	\]
	which completes the proof.
\end{proof}
\begin{thmm}[Descent]
	Algorithm \ref{alg:MVSfM} is a descent method, that is,
	\[
	\mathcal{L}(\mathbf{F}_{t+1}, \mathbf{G}_{t+1}) \leq \mathcal{L}(\mathbf{F}_{t}, \mathbf{G}_{t}).
	\]
\end{thmm}
\begin{proof}
	Note that
	\[
	\begin{aligned}
	\mathcal{L}(\mathbf{F}_{t+1}, \mathbf{G}_{t}) 
	=&\|\mathbf{X}-\mathbf{F}_{t+1}\mathbf{G}_t\|_F^2 +\lambda \sum_{i=1}^{k-1} \log(\sigma_i^2 (\mathbf{F}_{t+1})+\varepsilon) \\
	=&  \|\mathbf{X}-\mathbf{F}_{t+1}\mathbf{G}_t\|_F^2  + \lambda(\sum_{i=1}^{k-1} \log(\sigma_i^2 (\mathbf{F}_{t+1})+\varepsilon) \\
	&- \sum_{i=1}^{k-1} \log(\sigma_i^2 (\mathbf{F}_{t})+\varepsilon)) + \lambda\sum_{i=1}^{k-1} \log(\sigma_i^2 (\mathbf{F}_{t})+\varepsilon) \\
	=&  \|\mathbf{X}-\mathbf{F}_{t+1}\mathbf{G}_t\|_F^2 +\lambda \sum_{i=1}^{k-1} \log( \frac{\sigma_i^2 (\mathbf{F}_{t+1})+\varepsilon}{\sigma_i^2 (\mathbf{F}_{t})+\varepsilon}) \\
	&+ \lambda\sum_{i=1}^{k-1} \log(\sigma_i^2 (\mathbf{F}_{t})+\varepsilon)  \\
	\overset{(a)}{\leq}&  \|\mathbf{X}-\mathbf{F}_{t+1}\mathbf{G}_t\|_F^2 +\lambda \left( \sum_{i=1}^{k-1} \frac{\sigma_i^2 (\mathbf{F}_{t+1})}{\sigma_i^2 (\mathbf{F}_{t})+\varepsilon} - \sum_{i=1}^{k-1} \frac{\sigma_i^2 (\mathbf{F}_{t})}{\sigma_i^2 (\mathbf{F}_{t})+\varepsilon}\right) \\
	&+ \lambda\sum_{i=1}^{k-1} \log\left(\sigma_i^2 (\mathbf{F}_{t})\right)  \\
	\overset{(b)}{\leq}&  \|\mathbf{X}-\mathbf{F}_{t+1}\mathbf{G}_t\|_F^2 +\lambda  \text{Tr}(\mathbf{D}_t^\top \mathbf{F}_{t+1}^\top \mathbf{F}_{t+1}) - \lambda \text{Tr}(\mathbf{D}_t^\top \mathbf{F}_{t}^\top \mathbf{F}_{t}) \\
	&+ \lambda\sum_{i=1}^{k-1} \log\left(\sigma_i^2 (\mathbf{F}_{t}+\varepsilon)\right)  \\
	\overset{(c)}{\leq}& \|\mathbf{X}-\mathbf{F}_{t}\mathbf{G}_t\|_F^2 + \lambda\sum_{i=1}^{k-1} \log\left(\sigma_i^2 (\mathbf{F}_{t})+\varepsilon\right) \\
	=& \mathcal{L}(\mathbf{F}_{t}, \mathbf{G}_{t}),
	\end{aligned}
	\]
	where (a) uses numerical inequality $\log(x)\leq x - 1$; (b) uses the asymmetric version of von Neumann's trace inequality (Lemma \ref{lem:traceIneq}) which implies 
	$\sum_{i=1}^k \frac{\sigma_i^2 (\mathbf{F}_{t+1})}{\sigma_i^2 (\mathbf{F}_{t})+\varepsilon} \leq \text{Tr}(\mathbf{D}_t^\top \mathbf{F}_{t+1}^\top \mathbf{F}_{t+1})$ and uses the definition of $\mathbf{D}_t$ which implies $\text{Tr}(\mathbf{D}_t^\top \mathbf{F}_{t}^\top \mathbf{F}_{t})=\sum_{i=1}^{k-1} \frac{\sigma_i^2 (\mathbf{F}_{t})}{\sigma_i^2 (\mathbf{F}_{t})+\varepsilon}$; (c) uses  $\|\mathbf{X}-\mathbf{F}_{t+1}\mathbf{G}_t\|_F^2 +\lambda  \text{Tr}(\mathbf{D}_t^\top \mathbf{F}_{t+1}^\top \mathbf{F}_{t+1})\leq \|\mathbf{X}-\mathbf{F}_{t}\mathbf{G}_t\|_F^2 +\lambda \text{Tr}(\mathbf{D}_t^\top \mathbf{F}_{t}^\top \mathbf{F}_{t})$.
	
	Leveraging the definition of $\mathbf{G}_{t+1}$, it is easy to see $\mathcal{L}(\mathbf{F}_{t+1}, \mathbf{G}_{t+1}) \leq \mathcal{L}(\mathbf{F}_{t+1}, \mathbf{G}_{t})$, which completes the proof.
\end{proof}
%
%
%
%
%
%
\newtheorem{thmmm}{\bf Theorem}
\setcounter{thmmm}{6}
\begin{thmmm}[Condition of TI-LSDable]
	$\mathbf{K}$ is TI-LSDable if and only if rank$(\mathbf{H}\mathbf{K}\mathbf{H}) \leq k-1$.
\end{thmmm}
\begin{proof}
	On the one hand, TI-LSDable $\Rightarrow$ rank$(\mathbf{H}\mathbf{K}\mathbf{H}) \leq k-1$. The reason is that leveraging Theorem \ref{thm:jmlr13prop} we know if $\mathbf{K}$ is TI-LSDable, then $\mathbf{X}$ can be written as $\mathbf{X}=\mathbf{FG}^\top$. Thus, rank$(\mathbf{H}\mathbf{K}\mathbf{H}) \leq k-1$.
	
	On the other hand, if rank$(\mathbf{H}\mathbf{K}\mathbf{H}) \leq k-1$, then there exists $\mathbf{K}^\prime \in \mathbb{K}(\mathbf{X})$ such that $ \mathbf{K}^\prime = c_1\mathbf{P}^\top\mathbf{P}$. We prove this claim as follows. Note that
	\[
	\mathbf{F} = \mathbf{\bar{F}} + \mathbf{m}\mathbb{1}_k^\top = r\sqrt{k(k-1)}\mathbf{U}_{k-1}\mathbf{B}^\top + \mathbf{m}\mathbb{1}_k^\top. 
	\]
	Then we have
	\[
	\begin{aligned}
	\mathbf{X}^\top\mathbf{X} &= \left( \bar{\mathbf{X}} + \mathbf{m}\mathbb{1}_n^\top \right)^\top \left( \bar{\mathbf{X}} + \mathbf{m}\mathbb{1}_n^\top \right) \\
	&= \mathbf{P}^\top\mathbf{F}^\top\mathbf{FP} \\
	&= \mathbf{P}^\top\left( \bar{\mathbf{F}} + \mathbf{m}\mathbb{1}_k^\top \right)^\top \left( \bar{\mathbf{F}} + \mathbf{m}\mathbb{1}_k^\top \right)\mathbf{P} \\
	&= \mathbf{P}^\top\left( \bar{\mathbf{F}}^\top\bar{\mathbf{F}} + \bar{\mathbf{F}}^\top \mathbf{m}\mathbb{1}_k^\top + \mathbb{1}_k\mathbf{m}^\top \bar{\mathbf{F}} + \|\mathbf{m}\|_2^2 \mathbb{1}_k\mathbb{1}_k^\top \right)\mathbf{P}.
	\end{aligned}
	\]
	According to rank$(\mathbf{H}\mathcal{K}\mathbf{H}) \leq k-1$, let $\mathbf{U}_\bot \in \mathbb{R}^{d\times (d-k+1)}$ be the orthocomplement of $\mathbf{U}_{k-1}$ and $\mathbf{m}=\mathbf{U}_\bot \mathbf{a}$ where $\mathbf{a} \in \mathbb{R}^{d-k+1}$. That is
	\[
	\begin{aligned}
	\mathbf{X}^\top\mathbf{X} &= \mathbf{P}^\top\left( \bar{\mathbf{F}}^\top\bar{\mathbf{F}} + \bar{\mathbf{F}}^\top \mathbf{m}\mathbb{1}_k^\top + \mathbb{1}_k\mathbf{m}^\top \bar{\mathbf{F}} + \|\mathbf{m}\|_2^2 \mathbb{1}_k\mathbb{1}_k^\top \right)\mathbf{P} \\
	&= \mathbf{P}^\top\left( \bar{\mathbf{F}}^\top\bar{\mathbf{F}} + \|\mathbf{m}\|_2^2 \mathbb{1}_k\mathbb{1}_k^\top \right)\mathbf{P} \\
	&= \mathbf{P}^\top\left( r^2k(k-1)\mathbf{BU}_{k-1}^\top\mathbf{U}_{k-1}\mathbf{B}^\top + \|\mathbf{m}\|_2^2 \mathbb{1}_k\mathbb{1}_k^\top \right)\mathbf{P} \\
	&= \mathbf{P}^\top\left( r^2k(k-1)\mathbf{B}\mathbf{B}^\top + \|\mathbf{m}\|_2^2 \mathbb{1}_k\mathbb{1}_k^\top \right)\mathbf{P} \\
	&= \mathbf{P}^\top\left( r^2k(k-1)\mathbb{I}_k - r^2(k-1) \mathbb{1}_k\mathbb{1}_k^\top + \|\mathbf{m}\|_2^2 \mathbb{1}_k\mathbb{1}_k^\top \right)\mathbf{P}
	\end{aligned}
	\]
	Let $\|\mathbf{a}\|_2 = \|\mathbf{m}\|_2 =  r\sqrt{k-1}$. Then we have
	\[
	\begin{aligned}
	\mathbf{X}^\top\mathbf{X} &= r^2k(k-1)\mathbf{P}^\top\mathbf{P},
	\end{aligned}
	\]
	which completes the proof.
\end{proof}
\begin{thmmm}
	Assume kernel $\mathbf{K}\in\mathbb{R}^{n\times n}$ satisfies rank$(\mathbf{K}) \leq n-1$.
	If $\mathbf{G}^*$ is a solution of the TI-LSD in Problem \eqref{eq:TILSD} for kernel $\mathbf{K}$,
	then there exists $\mathbf{X}$ such that $\mathbf{K}=\mathbf{X}^\top\mathbf{X}$ and $\mathbf{G}^*$ is also a solution of S$k$M
	in Problem \eqref{eq:SfM}.
\end{thmmm}
\begin{proof}
	Let $\mathbf{M}$ be some kernel that is TI-LSDable. Note that
	\[
	\begin{aligned}
	&\|\mathbf{K} - \frac{1}{c}\mathbf{G}^\top\mathbf{G} \|_F^2\\
	=& \|\mathbf{K} - \mathbf{M} + \mathbf{M} - \frac{1}{c}\mathbf{G}^\top\mathbf{G} \|_F^2 \\
	= &\|\mathbf{K} - \mathbf{M} \|_F^2  + \|\mathbf{M} - \frac{1}{c}\mathbf{G}^\top\mathbf{G} \|_F^2 + 2\langle \mathbf{K} - \mathbf{M}, \mathbf{M} - \frac{1}{c}\mathbf{G}^\top\mathbf{G} \rangle.
	\end{aligned}
	\]
	Since we will optimize on $\mathbf{G}$ and $\mathbf{M}$ is LSDable, for any $\mathbf{M}$, we can choose $c$ and $\mathbf{G}$ such that $\mathbf{M}=\frac{1}{c}\mathbf{G}^\top\mathbf{G}$. Thus the objective can be written as
	\[
	\min_{\substack{\mathbf{K}\in\mathbb{K}(\mathbf{X}), \\ \mathbf{M} \text{ is LSDable} }}\|\mathbf{K} - \mathbf{M} \|_F^2.
	\]
	Assume $\mathbf{M} = \mathbf{Y}^\top\mathbf{Y}$. We always have $\mathbf{Y} = \overline{\mathbf{Y}} + \mathbf{s}\mathbb{1}_n^\top$, where $\mathbf{s} = \frac{1}{n}\mathbf{Y}\mathbb{1}_n$ and $\overline{\mathbf{Y}}\mathbb{1}_n = 0$. Using this decomposition, we rewrite above problem explicitly as
	\[
	\min_{\substack{ \mathbf{t},\mathbf{s},\overline{\mathbf{Y}} \\ \mathbf{Y}^\top\mathbf{Y} \text{ is LSDable}}} 
	\|(\mathbf{X}+\mathbf{t}\mathbb{1}_n^\top)^\top(\mathbf{X}+\mathbf{t}\mathbb{1}_n^\top) - (\overline{\mathbf{Y}}+\mathbf{s}\mathbb{1}_n^\top)^\top(\overline{\mathbf{Y}}+\mathbf{s}\mathbb{1}_n^\top) \|_F^2,
	\]
	where $\mathbf{X}\mathbb{1}_n = 0$. 
	Note that
	\[
	\begin{aligned}
	&\|(\mathbf{X}+\mathbf{t}\mathbb{1}_n^\top)^\top(\mathbf{X}+\mathbf{t}\mathbb{1}_n^\top) - (\overline{\mathbf{Y}}+\mathbf{s}\mathbb{1}_n^\top)^\top(\overline{\mathbf{Y}}+\mathbf{s}\mathbb{1}_n^\top) \|_F^2 \\
	=& \|(\mathbf{X}^\top\mathbf{X}-\overline{\mathbf{Y}}^\top\overline{\mathbf{Y}}) + \mathbb{1}_n(\mathbf{t}^\top \mathbf{X} - \mathbf{s}^\top\overline{\mathbf{Y}}) + (\mathbf{X}^\top\mathbf{t} - \overline{\mathbf{Y}}^\top\mathbf{s})\mathbb{1}^\top \\
	&+ (\|\mathbf{t}\|^2 - \|\mathbf{s}\|^2)\mathbb{1}_n\mathbb{1}_n^\top  \|_F^2 \\
	=& \| \mathbf{X}^\top\mathbf{X}-\overline{\mathbf{Y}}^\top\overline{\mathbf{Y}} \|_F^2 + \| \mathbb{1}_n(\mathbf{t}^\top \mathbf{X} - \mathbf{s}^\top\overline{\mathbf{Y}}) + (\mathbf{X}^\top\mathbf{t} - \overline{\mathbf{Y}}^\top\mathbf{s})\mathbb{1}^\top\\
	&+ (\|\mathbf{t}\|^2 - \|\mathbf{s}\|^2)\mathbb{1}_n\mathbb{1}_n^\top \|_F^2 \\
	=& \| \mathbf{X}^\top\mathbf{X}-\overline{\mathbf{Y}}^\top\overline{\mathbf{Y}} \|_F^2 + \| \mathbb{1}_n(\mathbf{t}^\top \mathbf{X} - \mathbf{s}^\top\overline{\mathbf{Y}}) + (\mathbf{X}^\top\mathbf{t} - \overline{\mathbf{Y}}^\top\mathbf{s})\mathbb{1}^\top \|_F^2\\
	&+ \|(\|\mathbf{t}\|^2 - \|\mathbf{s}\|^2)\mathbb{1}_n\mathbb{1}_n^\top \|_F^2,
	\end{aligned}
	\]
	where the second and third equality hold since $\mathbf{X}\mathbb{1}_n=\overline{\mathbf{Y}}\mathbb{1}_n = 0$.
	From rank$(\mathbf{X})\leq n-1$ and rank$(\overline{\mathbf{Y}})\leq k-1$, we can alway choose proper $\mathbf{s}$ and $\mathbf{t}$ such that $\mathbf{Y}^\top\mathbf{Y}$ is LSDable and the second and third term are equal to zero. In detail, we choose $\|\mathbf{s}\|=\|\mathbf{t}\|$ and $\mathbf{X}^\top\mathbf{t}=\overline{\mathbf{Y}}^\top\mathbf{s}=0$. Thus, the original problem becomes
	\[
	\min_{\text{rank}(\overline{\mathbf{Y}})\leq k-1}  
	\|(\mathbf{X}^\top\mathbf{X} - \overline{\mathbf{Y}}^\top\overline{\mathbf{Y}} \|_F^2.
	\]
	The $(k-1)$-truncated SVD of $\mathbf{X}$ is a solution of above problem. Using Theorem \ref{thm:sfmable} and $\mathbf{Y}^\top\mathbf{Y} = \frac{1}{c}\mathbf{G}\mathbf{G}^\top$, the proof completes.
\end{proof}
%
%
%
%
%
%
%
\subsection{Technical Lemmas}
\newtheorem{Lemma2}{\bf Lemma}
\begin{Lemma2}
	Let $\mathbf{x} \in \mathbb{R}^k$.
	If $\mathbb{1}_k^\top\mathbf{x}=0$ and $\|\mathbf{x}\|_2 \leq 1$, then we have $\|\mathbf{x}\|_\infty \leq \frac{\sqrt{k(k-1)}}{k}.$
\end{Lemma2}
\begin{proof}
	Let $i_\infty$ be the index such that $x_{i_\infty} = \max_{1\leq i \leq k} |x_i|$. Note that
	\[
	x_{i_\infty}^2 = \left( \sum_{i\neq i_\infty} x_i \right)^2 \leq (k-1)\sum_{i\neq i_\infty}x_i^2 \leq (k-1)(1-x_{i_\infty}^2),
	\]
	which indicates $x_{i_\infty}^2 \leq \frac{k-1}{k}$. The claim just follows.
\end{proof}
%
\begin{Lemma2}
	Following equivalence can be shown using the translation invariance of Problem \eqref{eq:lessConstrained}:
	\[
	\min_{\mathbf{F},\mathbf{G}\mathbb{1}_k = \mathbb{1}_n} \| \mathbf{X} - \mathbf{FG}^\top\|_F^2
	\Leftrightarrow
	\min_{\substack{\mathbf{F},\mathbf{G}\mathbb{1}_k = \mathbb{1}_n \\ \mathbf{s}, \text{rank}(\mathbf{F}) \leq k-1 }} \| \mathbf{X} - \mathbf{s}\mathbb{1}_n^\top- \mathbf{FG}^\top\|_F^2.
	\]
\end{Lemma2}
\begin{proof}
	Note that,
	\[
	\begin{aligned}
	\min_{\mathbf{F},\mathbf{G}\mathbb{1}_k = \mathbb{1}_n} \| \mathbf{X} - \mathbf{FG}^\top\|_F^2 
	&=
	\min_{\mathbf{s}, \mathbf{F},\mathbf{G}\mathbb{1}_k = \mathbb{1}_n} \| \mathbf{X} - \left( \mathbf{F}- \mathbf{s}\mathbb{1}_k^\top \right)\mathbf{G}^\top\|_F^2\\
	&\leq
	\min_{\substack{\mathbf{F},\mathbf{G}\mathbb{1}_k = \mathbb{1}_n \\ \mathbf{s}, \text{rank}(F) \leq k-1 }} \| \mathbf{X} - \mathbf{s}\mathbb{1}_n^\top- \mathbf{FG}^\top\|_F^2.
	\end{aligned}
	\]
	Assume that,
	\[
	\mathbf{F}_*, \mathbf{G}_* = \mathop{\arg\min}_{\mathbf{F},\mathbf{G}\mathbb{1}_k = \mathbb{1}_n}\| \mathbf{X} - \mathbf{FG}^\top\|_F^2.
	\]
	We have,
	\[
	\begin{aligned}
	\left\|\mathbf{X} - \mathbf{F}_*\mathbf{G}_*^\top\right\|_F^2
	&=
	\left\|\mathbf{X} - \left(\frac{1}{k}\mathbf{F}_*\mathbb{1}_k\right)\mathbb{1}_k^\top \mathbf{G}_*^\top - \left(\mathbf{F}_* - \frac{1}{k}\mathbf{F}_*\mathbb{1}_k\mathbb{1}_k^\top \right)\mathbf{G}_*^\top  \right\|_F^2\\
	&=\|\mathbf{X} - \hat{\mathbf{s}}\mathbb{1}_n^\top - \hat{\mathbf{F}}\mathbf{G}_*^\top  \|_F^2,
	\end{aligned}
	\]
	where $\hat{\mathbf{s}} := \frac{1}{k}\mathbf{F}_*\mathbb{1}_k, \hat{\mathbf{F}} := \mathbf{F}_* - \frac{1}{k}\mathbf{F}_*\mathbb{1}_k\mathbb{1}_k^\top$. 
	
	Note that $\hat{\mathbf{F}} \mathbb{1}_k = \mathbf{F}_*\mathbb{1}_k - \frac{1}{k}\mathbf{F}_*\mathbb{1}_k\mathbb{1}_k^\top\mathbb{1}_k
	=\mathbf{F}_*\mathbb{1}_k-\mathbf{F}_*\mathbb{1}_k=0$. Thus, $\text{rank}(\hat{\mathbf{F}}) \leq k-1$ and we have
	\[
	\begin{aligned}
	\min_{\mathbf{F},\mathbf{G}\mathbb{1}_k = \mathbb{1}_n} \| \mathbf{X} - \mathbf{FG}^\top\|_F^2 
	&=
	\left\|\mathbf{X} - \mathbf{F}_*\mathbf{G}_*^\top\right\|_F^2\\
	&=\|\mathbf{X} - \hat{\mathbf{s}}\mathbb{1}_n^\top - \hat{\mathbf{F}}\mathbf{G}_*^\top  \|_F^2\\
	&\geq\min_{\substack{\mathbf{F},\mathbf{G}\mathbb{1}_k = \mathbb{1}_n \\ \mathbf{s}, \text{rank}(F) \leq k-1 }} \| \mathbf{X} - \mathbf{s}\mathbb{1}_n^\top- \mathbf{FG}^\top\|_F^2,
	\end{aligned}
	\]
	which completes the proof.
\end{proof}
%
\begin{Lemma2}
	Let $\mathbf{F}\in\mathbb{R}^{d\times k}$ be the prototype matrix satisfying the constraint $\text{rank}(\mathbf{F}) \leq k-1$ and has the thin SVD $\mathbf{F}=\mathbf{U}\mathbf{\Sigma}\mathbf{V}^\top$, where $\mathbf{\Sigma}\in\mathbb{R}^{(k-1)\times(k-1)}$. Let $ \mathbf{\Phi} = \mathbf{V}\mathbf{\Sigma}^{-1}\mathbf{U}^\top \left(\mathbf{X}-\mathbf{s}\mathbb{1}_n^\top\right)$. Define
	\[
	\mathbf{G}_*^\top = \mathbf{\Phi} 
	+
	\frac{\mathbf{v}_\bot\left(\mathbb{1}_n^\top - \mathbb{1}_k^\top \mathbf{\Phi} \right) }{\mathbb{1}_k^\top\mathbf{v}_\bot },
	\]
	where $\mathbf{v}_\bot \in \mathbb{R}^k$ satisfies $\|\mathbf{v}_\bot\|_2 = 1$ and $\mathbf{V}^\top \mathbf{v}_\bot = 0$.
	Then
	\[
	\mathbf{G}_* \in \mathop{\arg\min}_{\mathbf{G}\mathbb{1}_k = \mathbb{1}_n} \| \mathbf{X} - \mathbf{s}\mathbb{1}_n^\top- \mathbf{FG}^\top\|_F^2.
	\]
\end{Lemma2}
\begin{proof}
	Note that $\mathbf{F} = \mathbf{U}\mathbf{\Sigma}\mathbf{V}^\top = \sum_{i=1}^{k-1}\sigma_i\mathbf{u}_i\mathbf{v}_i^\top$ and
	\[
	\| \mathbf{X} - \mathbf{s}\mathbb{1}_n^\top- \mathbf{FG}^\top\|_F^2 
	= 
	\sum_{i=1}^n \|\mathbf{x}_i - \mathbf{s} - \mathbf{F}\mathbf{g}_i \|_2^2
	\]
	Then we can write $\mathbf{g}_i = \mathbf{V}\mathbf{a}_i + \mathbf{v}_\bot b_i$, where $\mathbf{a}_i \in \mathbb{R}^{k-1}, b_i \in \mathbb{R}$ and 
	$\mathbf{v}_\bot \in \mathbb{R}^k,\mathbf{V}^\top\mathbf{v}_\bot=0$.  We solve $\mathbf{g}_i$ by solving $\mathbf{a}_i$ and $b_i$.
	Note that
	\[
	\begin{aligned}
	\sum_{i=1}^n \|\mathbf{x}_i - \mathbf{s} - \mathbf{F}\mathbf{g}_i \|_2^2
	&=
	\sum_{i=1}^n \|\mathbf{x}_i - \mathbf{s} - \mathbf{U}\mathbf{\Sigma}\mathbf{V}^\top\mathbf{V}\mathbf{a}_i - \mathbf{U}\mathbf{\Sigma}\mathbf{V}^\top \mathbf{v}_\bot b_i \|_2^2\\
	&=
	\sum_{i=1}^n \|\mathbf{x}_i - \mathbf{s} - \mathbf{U}\mathbf{\Sigma}\mathbf{a}_i\|_2^2.
	\end{aligned}
	\]
	Then we can obtain the optimal $\mathbf{a}_i$ as $\mathbf{a}_i^* = \mathbf{\Sigma}^{-1}\mathbf{U}^\top(\mathbf{x}_i-\mathbf{s})$.
	Using the constraint $\mathbf{G}\mathbb{1}_k=\mathbb{1}_n$, we have
	\[
	\mathbb{1}_k^\top \mathbf{g}_i = \mathbb{1}_k^\top \mathbf{V}\mathbf{\Sigma}^{-1}\mathbf{U}^\top(\mathbf{x}_i-\mathbf{s}) + \mathbb{1}_k^\top\mathbf{v}_\bot b_i = 1,
	\]
	which gives the optimal $b_i$ as
	\[
	b_i = \frac{1 - \mathbb{1}_k^\top \mathbf{V}\mathbf{\Sigma}^{-1}\mathbf{U}^\top(\mathbf{x}_i-\mathbf{s})}{\mathbb{1}_k^\top\mathbf{v}_\bot }.
	\]
	Then the optimal $\mathbf{g}_i$ writes
	\[
	\mathbf{g}_i^* = \mathbf{V} \mathbf{\Sigma}^{-1}\mathbf{U}^\top(\mathbf{x}_i-\mathbf{s}) +  \frac{1 - \mathbb{1}_k^\top \mathbf{V}\mathbf{\Sigma}^{-1}\mathbf{U}^\top(\mathbf{x}_i-\mathbf{s})}{\mathbb{1}_k^\top\mathbf{v}_\bot } \mathbf{v}_\bot,
	\]
	which completes the proof.
\end{proof}
%
%
%
%
\begin{Lemma2}
	Assume $\mathbf{F}=\mathbf{U\Sigma V}^\top$ such that $\mathbf{F}\mathbb{1}_k = 0$. Then,
	\[
	\log \det 
	\left[ \begin{array}{c}
	\mathbf{\Sigma V}^\top \\ \mathbb{1}_k^\top \end{array}
	\right] = \log(\sqrt{k}) +  \sum_{i=1}^{k-1} \log(\sigma_i(\mathbf{F})).
	\]
\end{Lemma2}
\begin{proof}
	Using $\det (\mathbf{A}\mathbf{A}^\top)= \det^2 \mathbf{A}$ and $\mathbf{V}^\top\mathbb{1}_k = 0$, we have
	
	\[
	\begin{aligned}
	\log \det 
	\left[ \begin{array}{c}
	\mathbf{\Sigma V}^\top \\ \mathbb{1}_k^\top \end{array}
	\right]
	=& 
	\frac{1}{2} \log \det \left(
	\left[ \begin{array}{c}
	\mathbf{\Sigma V}^\top \\ \mathbb{1}_k^\top \end{array}
	\right]
	\left[ \begin{array}{c}
	\mathbf{\Sigma V}^\top \\ \mathbb{1}_k^\top \end{array}
	\right]^\top
	\right) \\
	=&
	\frac{1}{2} \log \det \left(
	\left[ \begin{array}{c c}
	\mathbf{\Sigma}^2 & 0 \\ 0 & k \end{array}
	\right]
	\right) \\
	=&
	\frac{1}{2}\log \left( k\prod_{i=1}^{k-1}\sigma_i^2(\mathbf{F}) \right) \\
	=& \log\sqrt{k} + \sum_{i=1}^{k-1}\log \sigma_i(\mathbf{F}),
	\end{aligned}
	\]
	which completes the proof.
\end{proof}
%
\begin{Lemma2}\label{lem:rankXH}
	Given $\mathbf{s}\in\mathbb{R}^d$, we have rank$(\mathbf{XH}_n) \leq$ rank$(\mathbf{X}-\mathbf{s}\mathbb{1}_n^\top)$.
\end{Lemma2}
\begin{proof}
	Leveraging the rank-sum inequality Lemma \ref{lem:rankSum}, we have
	\[
	\text{rank}\left(\mathbf{X} - \frac{1}{n}\mathbf{X}\mathbb{1}_n\mathbb{1}_n^\top\right)
	\leq \text{rank}\left(\mathbf{X} - \mathbf{s}\mathbb{1}_n^\top\right) + \text{rank}\left(\mathbf{s}\mathbb{1}_n^\top - \frac{1}{n}\mathbf{X}\mathbb{1}_n\mathbb{1}_n^\top\right),
	\]
	where the equality holds only if $\text{range}\left( \mathbf{X} - \mathbf{s}\mathbb{1}_n^\top \right) \cap \text{range}\left(\mathbf{s}\mathbb{1}_n^\top - \frac{1}{n}\mathbf{X}\mathbb{1}_n\mathbb{1}_n^\top\right) = \{0\}$.
	
	Note that
	\[
	\text{range}\left( \mathbf{X} - \mathbf{s}\mathbb{1}_n^\top \right) \ni \sum_{i=1}^n \frac{\mathbf{x}_i - \mathbf{s}}{n} = \frac{1}{n}\mathbf{X}\mathbb{1}_n - \mathbf{s} \in \text{range}\left(\mathbf{s}\mathbb{1}_n^\top - \frac{1}{n}\mathbf{X}\mathbb{1}_n\mathbb{1}_n^\top\right),
	\]
	which concludes that the equality cannot hold. Thus
	$
	\text{rank}\left(\mathbf{XH}_n\right)
	< \text{rank}\left(\mathbf{X} - \mathbf{s}\mathbb{1}_n^\top\right) + 1,
	$
	which completes the proof.
\end{proof}
%
%
%
%
%
%
%
%
\begin{Lemma2}[\cite{von1937some}]\label{lem:traceIneq}
	Assume $\mathbf{A}\in\mathbb{R}^{n\times n}$ and $\mathbf{B}\in\mathbb{R}^{n\times n}$ are symmetric. Then
	\[
	\sum_{i=1}^n \sigma_{n-i+1}(\mathbf{A})\sigma_i(\mathbf{B}) \leq \text{Tr}(\mathbf{AB})\leq \sum_{i=1}^n \sigma_{i}(\mathbf{A})\sigma_i(\mathbf{B}).
	\]
\end{Lemma2}

\begin{Lemma2}[\cite{horn2013matrix}, 0.4.5.1]\label{lem:rankSum}
	The rank-sum inequality: If $\mathbf{A},\mathbf{B} \in \mathbb{R}^{m\times n}$, then
	\[
	\text{rank}(\mathbf{A}+\mathbf{B}) \leq \text{rank}(\mathbf{A}) + \text{rank}(\mathbf{B}),
	\]
	with equality {\em if and only if}
	$
	(\text{range}\ \mathbf{A}) \cap (\text{range}\ \mathbf{B}) = \{0\}$ and
	$
	(\text{range}\ \mathbf{A}^\top) \cap (\text{range}\ \mathbf{B}^\top) = \{0\}.
	$
\end{Lemma2}

\begin{Lemma2}[Mirsky's inequality, \cite{stewart1998perturbation}]\label{lem:mirsky}
	Let $\mathbf{A}\in\mathbb{R}^{n\times m}$. Then 
	\[
	\sum_{i=1}^n \left(\sigma_i(\mathbf{A}+\mathbf{E}) - \sigma_i(\mathbf{A}) \right)^2 \leq \|\mathbf{E}\|_F^2.
	\]
\end{Lemma2}
\ifCLASSOPTIONcaptionsoff
  \newpage
\fi



%
%
%

\bibliographystyle{IEEEtran}
\bibliography{softk}
%




\end{document}